\documentclass[sn-mathphys,Numbered]{sn-jnl}%

\usepackage{graphicx}%
\usepackage{multirow}%
\usepackage{amsmath,amssymb,amsfonts}%
\usepackage{mathtools}
\usepackage{amsthm}%
\usepackage{mathrsfs}%
\usepackage[title]{appendix}%
\usepackage{xcolor}%
\usepackage{textcomp}%
\usepackage{manyfoot}%
\usepackage{booktabs}%
\usepackage{algorithm}%
\usepackage{algorithmic}%
\usepackage{listings}%
\usepackage{hyperref}%

\usepackage{threeparttable}
\usepackage{booktabs}%
\usepackage{colortbl}
\definecolor{bgcolor}{rgb}{0.8,1,1}
\definecolor{bgcolor3}{rgb}{0.9, 0.17, 0.31}
\definecolor{bgcolor2}{rgb}{0.8,1,0.8}
\usepackage{makecell}
\usepackage{caption}
\usepackage{multirow}

\usepackage[capitalise]{cleveref}%

\newtheorem{theorem}{Theorem}

\newtheorem{corollary}{Corollary}

\newtheorem{assumption}{Assumption}
\newtheorem{lemma}{Lemma}%

\newcommand*{\R}{{\mathbb R}}
\newcommand*{\E}{{\mathbb E}}

\renewcommand{\|}{\parallel}

\newcommand\ec[2][]{\ensuremath{\mathbb{E}_{#1} \left[#2\right]}}
\newcommand\ecn[3][]{\ec[#1]{\norm{#2}^2_{#3}}}

\newcommand\ev[1]{\left \langle #1 \right \rangle}
\newcommand\br[1]{\left ( #1 \right )}
\newcommand\pbr[1]{\left \{ #1 \right \} }

\newcommand{\N}{\mathbb{N}}

\newcommand{\norm}[1]{\left\lVert#1\right\rVert}
\newcommand{\sqn}[1]{{\left\lVert#1\right\rVert}^2}
\newcommand{\abs}[1]{\left\lvert#1\right\rvert}

\newcommand{\D}{\mathcal{D}}
\newcommand{\eqdef}{\overset{\text{def}}{=}}

\newcommand{\avemm}{\frac{1}{M}\sum\limits_{m=1}^M}
\newcommand{\aveim}{\frac{1}{M}\sum_{i=1}^M}

\newcommand{\sigmaf}{\sigma_{\mathrm{dif}}}

\newcommand{\sm}[2]{\begin{smallmatrix}\item1\\\item2 \end{smallmatrix}}

\def\R{\mathbb{R}}

\newcommand\ddfrac[2]{\frac{\displaystyle \item1}{\displaystyle \item2}}

\usepackage[textsize=tiny]{todonotes}

\begin{document}

\title[]{Local Methods with Adaptivity via Scaling}

\author*[1,2,3]{\fnm{Savelii} \sur{Chezhegov}}\email{chezhegov.sa@phystech.edu}

\author[2]{\fnm{Sergey} \sur{Skorik}}

\author[4]{\fnm{Nikolas} \sur{Khachaturov}}

\author[5]{\fnm{Danil} \sur{Shalagin}}

\author[2]{\fnm{Aram} \sur{Avetisyan}}

\author[3]{\fnm{Martin} \sur{Takáč}}

\author[5]{\fnm{Yaroslav} \sur{Kholodov}}

\author[2,1,5]{\fnm{Aleksandr} \sur{Beznosikov}}

\affil[1]{\orgname{Moscow Institute of Physics and Technology}, \orgaddress{\city{Dolgoprudny}, \country{Russia}}}

\affil[2]{\orgname{Ivannikov Institute for System Programming}, \orgaddress{\city{Moscow}, \country{Russia}}}

\affil[3]{\orgname{Mohamed bin Zayed University of Artificial Intelligence}, \orgaddress{\city{Abu Dhabi}, \country{United Arab Emirates}}}

\affil[4]{\orgname{Russian-Armenian University}, \orgaddress{\city{Yerevan}, \country{Armenia}}}

\affil[5]{\orgname{Innopolis University}, \orgaddress{\city{Innopolis}, \country{Russia}}}



\maketitle

\begin{abstract}{}
The rapid development of machine learning and deep learning has introduced increasingly complex optimization challenges that must be addressed. Indeed, training modern, advanced models has become difficult to implement without leveraging multiple computing nodes in a distributed environment. Distributed optimization is also fundamental to emerging fields such as federated learning. Specifically, there is a need to organize the training process to minimize the time lost due to communication. A widely used and extensively researched technique to mitigate the communication bottleneck involves performing local training before communication. This approach is the focus of our paper. Concurrently, adaptive methods that incorporate scaling, notably led by Adam, have gained significant popularity in recent years. Therefore, this paper aims to merge the local training technique with the adaptive approach to develop efficient distributed learning methods. We consider the classical Local SGD method and enhance it with a scaling feature. A crucial aspect is that the scaling is described generically, allowing us to analyze various approaches, including Adam, RMSProp, and OASIS, in a unified manner. In addition to the theoretical analysis, we validate the performance of our methods in practice by training a neural network.

\end{abstract}

\section{Introduction}

\paragraph{Distributed optimization}
Today, there are numerous problem statements related to distributed optimization, with one of the most popular applications being in machine learning \cite{shalev2014understanding} and deep learning \cite{goodfellow2016deep}. The rapid advancement of artificial intelligence has enabled the resolution of a wide range of challenges, ranging from object classification,  
credit scoring to machine translation. 
Consequently, the complexity of machine learning problems has been increasing steadily, necessitating the processing of ever-larger data sets with increasingly large models \cite{arora2018optimization}. Therefore, it is now challenging to envision learning processes without the use of parallel computing and, by extension, distributed algorithms \cite{smith2018cocoa,mishchenko2019distributed,verbraeken2020survey,chraibi2019distributed,pirau2024preconditioning}.

In practice, various settings for distributed computing can be considered. One classical distributed setting involves optimization within a computing cluster. In this scenario, it is assumed that optimization occurs on a computational cluster, facilitating parallel computations to expedite the learning process and enabling communication between different workers. This cluster setting closely resembles collaborative learning, where, instead of a single large cluster, there is a network of users with potentially smaller computational resources. These resources could be virtually combined into a substantial computational resource capable of addressing the overarching learning problem.

It is worth noting that, in these formulations, it can be assumed that local data on each device is homogeneous—identical and originating from the same distribution—due to the artificial uniform partitioning of the dataset across computing devices. A more realistic setting occurs when data are inherently distributed among users, resulting in heterogeneity among local data across workers. An example of a problem statement within this setting is federated learning \cite{konevcny2016federated,
kairouz2021advances,karimireddy2020scaffold}, where each user's data is private and often has a unique nature, coming from heterogeneous distributions.

\paragraph{Communication bottleneck} 

A central challenge of distributed optimization algorithms is organizing communications between computing devices or nodes. Consequently, numerous aspects need consideration, ranging from the organization of the communication process to its efficiency. In particular, communication costs can be significant due to the large volumes of information transmitted, especially in contemporary problems. As a result, total communication time in distributed learning becomes a bottleneck that must be addressed.



Various techniques exist to address the issue of total communication complexity \cite{arjevani2015communication, alistarh2017qsgd, smith2018cocoa, beznosikov2020biased}. In this paper, we concentrate on one such technique: local steps. The core of this approach is to enable devices to perform a certain number of local computations without the need for communication. Consequently, this significantly reduces the frequency of communications, which can in turn lower the final complexity of algorithms.

\paragraph{Methods with local steps}

Local techniques have already been explored in various contexts. Among the first local method, Parallel SGD was proposed in \cite{mangasarian1995parallel} and has since evolved, appearing under new identities such as Local SGD \cite{stich2018local} and FedAvg \cite{mcmahan2017communication}. Additionally, there have been studies introducing various modifications, including momentum \cite{wang2019slowmo}, quantization \cite{basu2020qsparse, reisizadeh2020fedpaq}, and variance reduction \cite{liang2019variance, sharma2019parallel}. Nevertheless, new interpretations of local techniques continue to emerge. For instance, relatively recent works like SCAFFOLD \cite{karimireddy2020scaffold} and ProxSkip \cite{mishchenko2022proxskip}, along with their modifications, have been introduced. It is also worth mentioning the research path focusing on Hessian similarity \cite{shamir2014communication, hendrikx2020statistically, beznosikov2021distributed, kovalev2022optimal} of local functions, where local techniques are employed, but in these cases, most of the local steps are performed by a single/main device.

Although these approaches and techniques are widespread, they are all fundamentally based on variations of gradient descent, whether deterministic or stochastic. Recently, particularly in the field of machine learning, so-called adaptive methods have gained traction. These methods, which involve fitting adaptive parameters for individual components of a variable, have become increasingly popular due to research demonstrating their superior results in learning problems. This can be achieved by utilizing a scaling/preconditioning matrix that alters the vector direction of the descent. Among the most prominent methods incorporating this approach are Adam \cite{kingma2014adam}, RMSProp \cite{tieleman2012lecture}, and AdaGrad \cite{duchi2011adaptive}.

\paragraph{The choice of preconditioner}

The structure of the preconditioner can vary significantly. For instance, calculations based on the gradient at the current point, as exemplified by Adam and RMSProp, can be utilized. Alternatively, a scaling matrix structure based on the Hutchinson approximation \cite{bekas2007estimator,sadiev2024stochastic}, such as that employed in OASIS \cite{jahani2021doubly}, may be used. To enhance computation, recurrence relations (exponential smoothing) are typically introduced for the preconditioner. One of its most critical attributes is its diagonal form, attributed to the fact that using the preconditioner resembles the quasi-Newton \cite{dennis1977quasi, fletcher2013practical} method. Consequently, calculating the inverse matrix becomes necessary, and this task is simplified significantly for the diagonal form.

The aforementioned techniques -- local steps in distributed optimization and preconditioning -- are widely used and have practical significance, yet their integration has not been extensively explored. Therefore, we have defined the following objectives for our research:
\textit{
\begin{itemize}
    \item 
    Combine the two techniques to introduce a new local method with preconditioning updates.
\item 
Obtain a general convergence analysis of this method for a specific class of adaptive settings.
\end{itemize}
}



\section{Contributions and related works}

Our contributions are delineated into four main parts:
 \begin{itemize}
\item 
{\bf Formulation of a new local method:} This paper introduces a method that merges the Local SGD technique with preconditioning updates. In this approach, each device employs preconditioning to scale gradients on each node within the communication network. Devices perform multiple iterations to solve their local problems and, during a communication round, transmit information to the server, where the collected variables are averaged. Concurrently, the preconditioning matrix is updated and distributed by the server to all clients. The concept of Local SGD has been extensively explored in literature \cite{konevcny2016federated, stich2018local, khaled2020tighter, koloskova2020unified, beznosikov2022decentralized}, with \cite{khaled2020tighter} providing particularly tight and unimproved results \cite{glasgow2022sharp}, which we leverage for our analysis.

\item 
{\bf Unified assumption on the preconditioning matrix:} To facilitate the convergence analysis of our novel algorithm, the exact structure of the scaling matrix need not be specified. We introduce a classical general assumption that allows us to examine a specific class of preconditioners simultaneously. This approach aligns with previous studies \cite{sadiev2022stochastic,beznosikov2022scaled}, which have indicated that many adaptive methods, including Adam, RMSProp, and OASIS, adhere to this assumption.

\item 
{\bf Setting for Theoretical Analysis: }Diverging from prior research, we depart from certain assumptions, such as gradient boundedness and gradient similarity (\cite[Assumptions 2 and 3]{reddi2020adaptive}). This shift enables us to address a broader class of problems.

\item 
{\bf Interpretability of Results:} Prior research has examined the combination of the two heuristics mentioned above. Specifically, \cite{reddi2020adaptive} explored local methods with scaling, such as FedAdaGrad and FedAdam, but we identified discrepancies in their results. Our study elucidates why our approach is exempt from these shortcomings (see \cref{fedavg}). 
Our findings offer greater interpretability, despite focusing on a more generalized theory in terms of preconditioning.

 \end{itemize}

\section{Preliminaries, requirements and notations}
In distributed optimization, we solve an optimization problem in a form 
\begin{equation}
     \label{eq:optimization-problem}
     \min_{x \in \R^d} \pbr{ f(x) \eqdef \frac{1}{M} \sum_{m=1}^{M} f_m (x)},
\end{equation}
where $f_m(x) \eqdef \ec[z_m \sim \D_m]{f_m(x, z_m)}$ is the loss function for $m^{th}$-client, $m \in [M]$ and $\D_m$ is the distribution of the data for $m^{th}$-client, $x$ are parameters of the model. We also denote the solution of the problem \eqref{eq:optimization-problem} as $x_\ast$.

Given the diverse formulations of distributed learning, we differentiate between two scenarios: homogeneous (identical) and heterogeneous. Formally, in the homogeneous case, the equality of loss functions is guaranteed by the uniformity of the data: $f_1(x) = \ldots = f_m(x)$. Conversely, the heterogeneous case arises when such equality does not hold.

To prove convergence, we introduce classical assumptions \cite{nesterov2003introductory, nemirovskij1983problem} for objective functions: smoothness, unbiasedness and bounded variance, and smoothness of the stochastic function. Our analysis, based on \cite{khaled2020tighter}, adopts the same set of assumptions.
 

\begin{assumption}
    \label{asm:convexity-and-smoothness}
     Assume that each $f_m$ is $\mu$-strongly convex for $\mu \geq 0$ and $L$-smooth. That is, for all $x, y \in \R^d$
     \begin{align*}
          \frac{\mu}{2} \sqn{x - y} \leq&\ f_m (x) - f_m (y) - \ev{\nabla f_m (y), x - y}
          \leq \frac{L}{2} \sqn{x - y}.
     \end{align*}
      Also we define $\kappa \eqdef \frac{L}{\mu}$, the condition number.
\end{assumption}
Next, we state two sets of assumptions about the problem's stochastic nature, leading to varying convergence rates.
\begin{assumption}
    \label{asm:uniformly-bounded-variance}
    Assume that $f_m$ satisfies 
    following:
  for all $x \in \R^d$ with $z \sim {\D_m}$ 
  drawn i.i.d.  according to a distribution ${\D_m}$:
    \begin{eqnarray*}
        &\ec[z \sim \D_m]{\nabla f_m(x, z)} = \nabla f_m(x), \\
        &\ecn[z \sim \D_m]{\nabla f_m(x, z) - \nabla f_m(x)}{} \leq \sigma^2.
    \end{eqnarray*}
\end{assumption}
Assumption \ref{asm:uniformly-bounded-variance} traditionally controls stochasticity but often does not apply to finite-sum problems (for $\mu > 0$ strongly convex objectives \cite{nguyen2018sgd}). To encompass such cases, we introduce the following assumption:
\begin{assumption}
    \label{asm:finite-sum-stochastic-gradients}
    Assume that $f_m(\cdot, z): \R^d \to \R$ is almost-surely $L$-smooth and $\mu$-strongly convex.
\end{assumption}
Additionally, to yield more precise results in heterogeneous scenarios, we introduce the following definition:
\[ \sigmaf^2 \eqdef \frac{1}{M} \sum_{m=1}^{M} \ecn[z_m \sim \D_m]{\nabla f_m (x_\ast, z_m)}{}. \]

\section{Preconditioning meets local method}

We are now prepared to present our algorithm, introduce a unified assumption for the preconditioner along with its properties, and discuss theoretical results for convergence within two distinct regimes.

\begin{algorithm}
   \caption{Stochastic Adaptive Vehicle with Infrequent Communications}
   \label{alg:local_sgd_with_preconditioner}
\begin{algorithmic}[1]
  \REQUIRE step-size $\gamma > 0$, initial vector $x_0 = x_0^m$ for all $m \in [M]$, synchronization timesteps $t_0 = 0, t_1, t_2, \ldots$.
   \FOR{$t=0,1,\dotsc$}
      \FOR{$m=1,\dotsc, M$ in parallel}
         \IF{$t = t_p$ for some $p$} 
            \STATE update the matrix $\textcolor{blue}{\hat{D}^{t_p}}$
        \ENDIF
         \STATE sample $z_m \overset{\text{i.i.d.}}{\sim} \D_m$
         \IF{data is identical}
            \STATE compute $\nabla f_m(x_t^m, z_m) = \nabla f(x_t^m, z_m)$ such that $\ec{\nabla f(x_t^m, z_m) \mid x_t^m} = \nabla f(x_t^m)$
         \ELSE
            \STATE compute $\nabla f_m(x_t^m, z_m)$, such that $\ec{\nabla f_m(x_t^m, z_m)\mid x_t^m}=\nabla f_m (x_t^m)$
         \ENDIF
         \STATE $x_{t+1}^m=
         \begin{cases}
         \frac{1}{M}\sum_{j=1}^M (x_t^j - \gamma \textcolor{blue}{(\hat{D}^{t_p})^{-1}}\nabla f_j(x_t^j, z_j)), & \text{ if } t = t_p \text { for some } p \in \N \\
         x_t^m - \gamma \textcolor{blue}{(\hat{D}^{t_p})^{-1}}\nabla f_m(x_t^m, z_m), & \text{ otherwise}
         \end{cases}$
      \ENDFOR
   \ENDFOR   
\end{algorithmic}
\end{algorithm}
\subsection{SAVIC}

Let us describe the Algorithm~\ref{alg:local_sgd_with_preconditioner},  
which is based on Local SGD. Each client maintains its own variable 
$x_t^m$, where $m \in [M]$ 
represents the client index,
$t$ denotes the iteration number, and  
$f_m$ 
is the loss function of the 
$m^{th}$-client. 
Additionally, we define a series of time moments 
$t_1, t_2, \ldots$, designated for communication. Depending on whether the current point in time aligns with a communication moment, a client either performs local steps or transmits its current information to the server, where averaging occurs. The innovation introduced in this algorithm is the matrix 
$\hat{D}^{t_p}$, with $t_p$
indicating one of the synchronization iterations. In the algorithm described, this modification acts as a preconditioner, highlighted in blue for emphasis. It is crucial to note that the matrix is updated exclusively during synchronization moments and remains consistent across all clients. The rationale behind these facts will be elucidated later when we delve into the properties of the preconditioner.

\subsection{Preconditioning}
In practice, we often encounter various recurrence relations between preconditioners for two adjacent iterations. One rule for updating the scaling matrix is given by the following expression:
\begin{align}
    \label{eq:square-update}
    (D^{t})^2 = \beta_t(D^{t-1})^2 + (1 - \beta_t)(H^{t})^2,
\end{align}
where $\beta_t \in \left[0, 1\right]$ is a preconditioning momentum parameter and $(H^t)^2$ is a diagonal matrix. This update mechanism is observed in Adam-based methods, where  
$(H^t)^2 = \textbf{diag}\left(\nabla f(x_t, z_t) \odot \nabla f(x_t, z_t)\right)$,
here $z_t$ is an independent random variable.
The original Adam \cite{kingma2014adam} has $\beta_t = \frac{\beta - \beta^{t+1}}{1 - \beta^{t+1}}$ 
or an earlier method, RMSProp \cite{tieleman2012lecture} has $\beta_t \equiv \beta$. 
Furthermore, the update rule \eqref{eq:square-update} is also applicable to AdaHessian \cite{yao2021adahessian}, which relies on Hutchinson method. 
For this, we need to select the momentum parameter as 
$\beta_t = \frac{\beta - \beta^{t+1}}{1 - \beta^{t+1}}$
  and set 
 $(H^t)^2 = \textbf{diag}\left(v_t \odot \nabla^2 f(x_t, z_t) v_t \right)^2$
 , where 
 $v_t$ has i.i.d. elements that 
  follow the Rademacher distribution.

Alternatively, the update for the preconditioning matrix can be represented as
\begin{align}
    \label{eq:linear-update}
    D^{t} = \beta_tD^{t-1}+ (1 - \beta_t)H^{t}.
\end{align}
This approach is also commonly used, for instance in methods like OASIS \cite{jahani2021doubly}, 
where $\beta_t \equiv \beta$ and 
$H^t = \textbf{diag}\left(v_t \odot \nabla^2 f(x_t, z_t) v_t \right)$.

Despite the presence of Hessians in AdaHessian and OASIS, computing the matrix of second derivatives is unnecessary; it suffices to compute the gradient 
$\nabla f(x_t, z_t)$
 and then calculate the gradient of 
 $\langle \nabla f(x_t, z_t), v_t \rangle$ (i.e., we just need to perform Hessian-vector product).

In practice, to avoid division by zero, a positive definite matrix is typically used. This can be achieved by the following modification to the matrix  $D^t$:
\begin{align}
    \label{eq:positive-definite}
    (\hat{D}^t)_{ii} = \max\{\alpha, |D^t_{ii}|\}.
\end{align}
Alternatively, the scaling matrix can be modified by adding 
$\alpha$ to each diagonal element: 
$(\hat{D}^t)_{ii} = |D^t_{ii}| + \alpha$
preserving the core idea of ensuring the preconditioner is positive definite.

To summarize the approaches, we make the following assumption:
\begin{assumption}
    \label{asm: preconditioner property}
    Assume that $D^{0}$ and $H^t$  
    satisfy the expression below with some
     $\alpha > 0$ and $\Gamma \geq \alpha$ for all $t$:
    \begin{align*}
        \alpha I \preceq D^{0} \preceq \Gamma I, \quad
        \alpha I \preceq H^{t} \preceq \Gamma I,
    \end{align*}
    where $I$ is an identity matrix.
\end{assumption}
From the above assumption, the following lemma implies:
\begin{lemma}[Lemma 1, Beznosikov et al., \cite{beznosikov2022scaled}]
\label{lemma: from-spp}
    Let us assume that $D^0$, and for all $t$ the $H^t$ is diagonal matrices with elements not greater than $\Gamma$ in absolute value. Then for matrices $\hat D^t$ obtained by rules \eqref{eq:square-update} -- \eqref{eq:positive-definite}, the following holds:
\begin{enumerate}
    \item 
$\hat D^t$ are diagonal matrices with non-negative elements and $\alpha I \preceq \hat{D}^{t} \preceq \Gamma I$;

    \item 
 $\hat{D}^{t+1} \preceq \left(1 + \frac{(1-\beta_{t+1})\Gamma^2 }{2\alpha^2}\right) \hat{D}^{t}$ for \eqref{eq:square-update};

    \item 
 $\hat{D}^{t+1} \preceq \left(1 + \frac{2(1-\beta_{t+1})\Gamma }{\alpha}\right) \hat{D}^{t}$ for  \eqref{eq:linear-update}.
\end{enumerate}
\end{lemma}
It is straightforward to demonstrate (refer to \cref{table_ef} and for further details, Lemma 2.1, Sections B, C in \cite{sadiev2022stochastic}) that classical and well-known preconditioners meet the conditions of Lemma \cref{lemma: from-spp}. The results of Lemma \cref{lemma: from-spp} have already been validated in \cite{beznosikov2022scaled}. However, for a clear convergence analysis, we need to formulate the following corollary:
\begin{corollary}
    \label{cor:equal-convergence}
    Suppose $\{\hat{x}_t\}_{t=0}$ are average points generated by \cref{alg:local_sgd_with_preconditioner}. Moreover, for any $t$ we have the scaling matrix $\hat{D}^{t}$ respectively. Hence, according to update rules \eqref{eq:square-update}, \eqref{eq:linear-update} and \eqref{eq:positive-definite}, we get
    \begin{align*}
        \sqn{\hat{x}_{t+1} - x_\ast}_{\hat{D}^{t+1}} \leq (1 + (1 - \beta_{t+1})C)\sqn{\hat{x}_{t+1} - x_\ast}_{\hat{D}^t},
    \end{align*}
    where $C$ depends on the preconditioner update setting.
    In particular, choosing $\beta_{t+1}$ in a certain way for each setting allows to claim the next result:
    \begin{align*}
        \sqn{\hat{x}_{t+1} - x_\ast}_{\hat{D}^{t+1}} \leq \left(1 + \frac{\gamma\mu}{2\Gamma}\right)\sqn{\hat{x}_{t+1} - x_\ast}_{\hat{D}^{t}},
    \end{align*}
      where $C =
    \begin{cases}
        \frac{\Gamma^2}{2\alpha^2} \text{ for \eqref{eq:square-update} }\Rightarrow \beta_{t+1} \geq 1 - \frac{\gamma\mu\alpha^2}{\Gamma^3},\\
        \frac{2\Gamma}{\alpha} \text{ for \eqref{eq:linear-update} } \Rightarrow \beta_{t+1} \geq 1 - \frac{\gamma\mu\alpha}{4\Gamma^2},
    \end{cases}$\\
    and $\sqn{x}_{\hat{D}^{t}}$ is the squared norm induced by the matrix, i.e. $\sqn{x}_{\hat{D}^{t}} = \ev{x, {\hat{D}^{t}}x}$.
\end{corollary}
\renewcommand{\arraystretch}{2}
\begin{table*}[!h]
    \centering
\captionof{table}{$\Gamma$ for various preconditioners. $G$ is the upper bound on the gradient norm. The presence of $G$ is typical for analysis of RMSProp and Adam \cite{defossez2020simple}.}
    \label{table_ef}   
    \small
  \begin{threeparttable}
    \begin{tabular}{ccc}
    \toprule
    \textbf{Method} & \textbf{$\Gamma$}  \\
    \midrule
    OASIS & $\sqrt{d}L$ \\\hline
    RMSProp & $G$\\\hline
    Adam  & $G$ \\
    \bottomrule
    \end{tabular}   
     \begin{tablenotes}
     \item 
\end{tablenotes}   
    \end{threeparttable}
\vspace{-0.3cm}
\label{table}
\end{table*}

\subsection{Convergence analysis}
\label{conv-analysis}
In the following subsections, we present the convergence results of  \cref{alg:local_sgd_with_preconditioner} for different settings: identical and heterogeneous data.
\subsubsection{Identical data}
Below is the main result for the identical data case.
\begin{theorem}
    \label{thm:ident_convergence_theorem}
    Suppose that Assumptions~\ref{asm:convexity-and-smoothness}, \ref{asm:uniformly-bounded-variance} and \ref{asm: preconditioner property} hold with $\mu > 0$. Then for \cref{alg:local_sgd_with_preconditioner} with identical data and a constant stepsize $\gamma > 0$ such that $\gamma \leq \frac{\alpha}{4L}$, and $H \geq 1$ such that $\max_{p} \abs{t_p - t_{p+1}} \leq H$, for all $T$ we have
     \begin{align*}
         \begin{split}
          \ecn{\hat{x}_T - x_\ast}{} = \mathcal{O}\left( \left(1 - \frac{\gamma\mu}{2\Gamma} \right)^{T}\frac{\Gamma}{\alpha} \sqn{x_0 - x_\ast} + \frac{\gamma \Gamma \sigma^2}{\alpha^2\mu M} + \frac{ L \gamma^2 \Gamma \br{H - 1} \sigma^2}{\mu\alpha^3}\right),
         \end{split}
     \end{align*}
      where $\hat{x}_t \eqdef \avemm x_t^m$.
\end{theorem} 
Using this theorem and properly selecting the parameters, we obtain the following result, which ensures convergence.
\begin{corollary}
    \label{clr:identical_convergence}
    If we choose $\gamma = \frac{\Gamma}{\mu a}$ with $a = 4\hat{\kappa} + t, t > 0$, where $\hat{\kappa} = \frac{L\Gamma}{\mu \alpha}$, and $T = 4a\log a$, then 
    substituting it into \cref{thm:ident_convergence_theorem} 
 and using the fact that $1 - x \leq \exp{(-x)}$, we obtain:
    \begin{equation*}
        \ecn{\hat{x}_T - x_\ast}{} = \tilde{\mathcal{O}} \br{ \frac{\Gamma\sqn{x_0 - x_\ast}}{\alpha T^2} + \frac{\Gamma \sigma^2}{\alpha \mu^2 M T} + \frac{\kappa \Gamma \sigma^2 (H-1)}{\mu^2 T^2 \alpha}},
     \end{equation*}
     where $\hat{x}_t \eqdef \avemm x_t^m$ and $\tilde{\mathcal{O}} (\cdot)$ omits polylogarithmic and constant factors.
\end{corollary}
\subsubsection{Heterogeneous data}
Next, we show a convergence guarantees for heterogeneous case.
\begin{theorem}
    \label{thm:hetero_convergence_theorem}
    Suppose that Assumptions~\ref{asm:convexity-and-smoothness}, \ref{asm:finite-sum-stochastic-gradients} and \ref{asm: preconditioner property} hold. 
    Then for \cref{alg:local_sgd_with_preconditioner} with heterogeneous setting, $M \geq 2$, $\max_{p} \abs{t_p - t_{p+1}} \leq H$, $\gamma > 0$ such that $\gamma \leq \frac{\alpha}{10(H-1)L}$, we have
    \begin{align*}
        \begin{split}
            \E\big[f(\bar{x}_{T-1}) - f(x_\ast)\big] \leq&\ \br{1 - \frac{\gamma \mu}{2\Gamma}}^T\frac{\Gamma\sqn{x_{0} - x_\ast}}{\gamma} + \gamma\sigmaf^2\br{\frac{9 (H-1) }{2\alpha}  + \frac{8}{M\alpha}},
        \end{split}
    \end{align*}
    where $\hat{x}_t \eqdef \avemm x_t^m$, $w_t \eqdef \br{1 - \frac{\gamma\mu}{2\Gamma}}^{-(t+1)}$, $W_{T-1} \eqdef \sum\limits_{t=0}^{T-1} w_t$ and $\bar{x}_{T-1} \eqdef \frac{1}{W_{T-1}}\sum\limits_{t=0}^{T-1}w_t \hat{x}_t$. 
\end{theorem}
\noindent Next corollary presents the convergence of \cref{alg:local_sgd_with_preconditioner} in heterogeneous case. 
\begin{corollary}
    \label{corollary:wc-noniid-unbounded-var}
    Choosing $\gamma$ as $\min{\left(\frac{\alpha}{10(H-1)L}, \frac{2\Gamma}{\mu T}\ln{\br{\max{\br{2, \frac{\mu^2 \sqn{x_{0} - x_\ast} T^2}{4\Gamma c}}}}}\right)}$, where $c \eqdef \sigmaf^2\br{\frac{9 (H-1) }{2\alpha}  + \frac{8}{M\alpha}}$ in the Theorem~\ref{thm:hetero_convergence_theorem}, we claim the result for the convergence rate: 
    \begin{align*}
        \begin{split}
            &\E\big[f(\bar{x}_{T-1}) - f(x_\ast)\big] \\ &= \tilde{\mathcal{O}}\br{\frac{(H-1)L\Gamma}{\alpha} \sqn{x_0 - x_\ast} \exp{\br{-\frac{\mu T \alpha}{\Gamma(H-1)L}}} + \frac{\Gamma \sigmaf^2 }{\alpha \mu T} \br{(H-1)   + \frac{1}{M}}}.
        \end{split}
    \end{align*}
    
\end{corollary}

\section{Discussion}
In this section, we discuss the results obtained by  
\cref{alg:local_sgd_with_preconditioner}
and qualitatively compare these results with existing estimates for 
algorithms of a similar structure.

\subsection{Interpretation of our results}
For a comprehensive understanding of the results obtained, it's essential to refer to both the theoretical analysis (see Section \ref{conv-analysis}) and the experiments conducted (see Section \ref{exper}).
\begin{itemize}
    \item \textbf{Preservation of analysis structure.} 
    The structure of the estimates obtained during our analysis remains consistent with that found in \cite{khaled2020tighter}, which served as the foundation for our analysis. Given that the estimates in the original paper \cite{khaled2020tighter} are shown to be optimal, as demonstrated in  \cite{glasgow2022sharp}, our estimates also achieve optimal performance within a specific class of preconditioning matrices.

    \item \textbf{Boundary behavior.} 
    The primary distinction between our analysis and that of the unscaled version in \cite{khaled2020tighter} lies in the introduction of constants $\alpha$ and $\Gamma$ into our estimates. The impact of $\Gamma$ is generally not significant, often becoming apparent through certain assumptions or lemmas, as illustrated in  \cite{sadiev2022stochastic}. 
    However, $\alpha$ represents a parameter that can be adjusted in practice, similar to implementations in Adam or OASIS. 
    Thus, the sensitivity of our estimates to this parameter, which is typically quite small, is notably significant.    
     
    \item \textbf{The relationship between  experiments and theory.} 
    The theoretical findings suggest a less favorable convergence rate for our method compared to classical Local SGD, attributed to an additional multiplicative factor $\frac{\Gamma}{\alpha}$ in our estimates. Conversely, experimental results indicate an enhanced convergence rate for our method over Local SGD. This discrepancy arises because our theorems rely on a unified assumption regarding the preconditioning matrix, whereas experiments employ specific scaling matrix structures. These structures, when incorporated into the theoretical analysis, could potentially reduce the algorithm's complexity. Our findings do not necessarily indicate a theoretical improvement over existing methods; rather, they confirm that convergence is achievable with the incorporation of adaptive structures through scaling. This opens avenues for future research, particularly in the exploration of specific types of preconditioning matrices, an area where a significant gap exists across all adaptive methods employing scaling.

\end{itemize}

\subsection{Discussion of results from \cite{reddi2020adaptive}}\label{fedavg}

In this section, we compare our approach with the algorithm (\cref{alg:fall}) developed in \cite{reddi2020adaptive}.
\begin{algorithm*}
    \caption{FedAdaGrad}\label{alg:fall}
\begin{algorithmic}[1]
\STATE initialization: $x_0, v_{-1} \geq \tau^2$, decay parameters $\beta_1, \beta_2 \in [0,1)$
\FOR {$t = 0, \cdots, T-1$}
    \STATE sample subset $\mathcal{S}$ of clients
    \STATE $x^t_{i,0} = x_{t}$
    \FOR{each client $i \in \mathcal{S}$ \textbf{in parallel}}
        \FOR {$k = 0, \cdots, K-1$}
            \STATE compute an unbiased estimate $g_{i,k}^t$ of $\nabla f_i(x^t_{i,k})$
            \STATE $x_{i,k+1}^t = x_{i, k}^t - \eta_l g_{i,k}^t$
        \ENDFOR
        \STATE $\Delta_i^t = x^t_{i, K} - x_{t}$
    \ENDFOR
    \STATE $\Delta_t = \frac{1}{|\mathcal{S}|} \sum_{i \in \mathcal{S}} \Delta_i^t$
    \STATE $m_t = \beta_1m_{t-1} + (1-\beta_1)\Delta_t$
    \STATE $v_t = v_{t-1} + \Delta_{t}^2$
    \STATE $x_{t+1} = x_{t} + \eta \frac{m_t}{\sqrt{v_t} + \tau}$
\ENDFOR
\end{algorithmic}
\end{algorithm*}
The following assumptions were made in the paper:
\begin{assumption}
\label{asp:lipschitz}
Assume that each $f_m$ is $L$-smooth. That is, for all $x, y \in \R^d$
\begin{align*}
    \norm{\nabla f_m(x) - \nabla f_m(y)} \leq L \norm{x - y}.
\end{align*}
\end{assumption} 

\begin{assumption}
\label{asp:variance}
Assume that $\{f_m\}_{m=1}^M$ satisfy next expressions with $z_m \sim \D_m$ 
$$\mathbb{E}[\sqn{\nabla [f_m(x,z_m)]_j - [\nabla f_m(x)]_j}] \leq \sigma_{l,j}^2, \qquad \text{(local variance)}$$
$$\frac{1}{M} \sum_{m=1}^M \sqn{[\nabla f_m(x)]_j - [\nabla f(x)]_j} \leq \sigma_{g,j}^2, \qquad \text{ (global variance)}$$ 
for all $x \in \mathbb{R}^d$ and $j \in [d]$.
\end{assumption}

\begin{assumption}
Assume that $f_m(x,z)$ have $G$-bounded gradients i.e., for any $m \in [M]$, $x \in \mathbb{R}^d$ and $z \sim \D_m$  
$$|[\nabla f_m(x,z)]_j| \leq G,$$  for all $j \in [d]$.
\label{asp:bounded-grad}
\end{assumption} 
Based on these assumptions, the authors of \cite{reddi2020adaptive} present the following results.
\begin{theorem}\label{thm:fadagrad_conv}
Suppose that Assumptions~\ref{asp:lipschitz}, \ref{asp:variance} and \ref{asp:bounded-grad} are satisfied. Then, for Algorithm \ref{alg:fall} with $\sigma^2 \eqdef \sigma_{l}^2 + 6K\sigma_{g}^2$ and $\eta_l$ such that
\begin{align}
\label{eq:cond-eta}
\eta_l \leq 
\begin{cases}
    \frac{1}{16K} \min \left\{\frac{1}{L}, \frac{1}{T^{1/6}} \left[\frac{\tau}{120L^2G} \right]^{1/3} \right\}, \\
    \frac{1}{16K} \min\left\{ \frac{\tau \eta L}{2G^2}, \frac{\tau}{4L\eta}, \frac{1}{T^{1/4}} \left[\frac{\tau^2}{GL\eta} \right]^{1/2} \right\},
\end{cases}
\end{align}
it follows that for all $T$ 
$$
\min_{0 \leq t \leq T-1} \E\|\nabla f(x_t)\|^2 \leq \mathcal{O}\left(\left[\frac{G}{\sqrt{T}} + \frac{\tau}{\eta_lKT}\right] \left( \Psi + {\Psi}_{\mathrm{var}} \right)\right),
$$
where
\begin{align*}
\Psi &= \frac{f(x_0) - f(x^*)}{\eta} +  \frac{5\eta_l^3K^2L^2T}{2\tau} \sigma^2,
\\
 {\Psi}_{\mathrm{var}} &=  \frac{2\eta_lKG^2 + \tau\eta L}{\tau^2} \left[ \frac{2\eta_l^2KT}{m} \sigma_{l}^2 + 10\eta_l^4K^3L^2T \sigma^2\right].
\end{align*}
\end{theorem}
Using Theorem \ref{thm:fadagrad_conv}, 
we can set $\eta = \text{const}$. 
Let us fix $T > 0$ and consider the case when $\tau$ tends to zero. 
Without loss of generality, 
according to \eqref{eq:cond-eta},
we get $\eta_l \sim \tau$. 
Substituting such $\eta_l = \tau \overline{\eta}_l$, we obtain:
\begin{align*}
     \min_{0 \leq t \leq T-1} \E\|\nabla f(x_t)\|^2 &\leq \ \mathcal{O}\bigg(\left[\frac{G}{\sqrt{T}} + \frac{\tau}{\eta_lKT}\right] \left( \Psi + {\Psi}_{\mathrm{var}} \right)\bigg) \\ =&\ \mathcal{O}\bigg(\left[\frac{G}{\sqrt{T}} + \frac{1}{\overline{\eta}_lKT}\right] \bigg(\frac{f(x_0) - f(x^*)}{\eta} +  \frac{5\tau^2\overline{\eta}_l^3L^2K^2T}{2} \sigma^2 \\
     &+ \frac{2\overline{\eta}_lKG^2 + \eta L}{\tau} \left[ \frac{2\tau^2\overline{\eta}_l^2KT}{m} \sigma_{l}^2 + 10\tau^4\overline{\eta}_l^4K^3L^2T \sigma^2\right] \bigg)\bigg) \\
     =&\ \mathcal{O}\bigg(\left[\frac{G}{\sqrt{T}} + \frac{1}{\overline{\eta}_lKT}\right] \bigg(\frac{f(x_0) - f(x^*)}{\eta} + \tau\bigg( \frac{5\tau\overline{\eta}_l^3L^2K^2T}{2} \sigma^2 \\
     &+ (2\overline{\eta}_lKG^2 + \eta L) \left[ \frac{2\overline{\eta}_l^2KT}{m} \sigma_{l}^2 + 10\tau^2\overline{\eta}_l^4K^3L^2T \sigma^2\right] \bigg)\bigg)\bigg) \\
     =&\ \mathcal{O}\bigg(\left[\frac{G}{\sqrt{T}} + \frac{1}{\overline{\eta}_lKT}\right] \bigg(\frac{f(x_0) - f(x^*)}{\eta}\bigg)\bigg)
\end{align*}
because  $\tau$ tends to zero.

Hence, the result appears to be independent of $\sigma_l^2$, suggesting the algorithm operates at any noise level, which seems impractical.
This issue can be related from an error in the proof of Theorem 1 in \cite{reddi2020adaptive}. 
Specifically, at the end of Theorem 1, page 18, the multiplication in lines 3-4 introduces $\tau^3$ in the denominator, whereas the final estimate incorrectly lists $\tau^2$. Therefore, a more accurate representation of ${\Psi}_{\mathrm{var}}$ would therefore be:
\[
 {\Psi}_{\mathrm{var}} =  \frac{2\eta_lKG^2 + \tau\eta L}{\textcolor{orange}{\tau^3}} \left[ \frac{2\eta_l^2KT}{m} \sigma_{l}^2 + 10\eta_l^4K^3L^2T \sigma^2\right].
\]\noindent
However, even after addressing this error, the analysis still faces challenges. With Algorithm 2 allowing $\beta_1=0$, and by substituting $v_{-1} = 1$ (feasible as $\tau$ approaches zero),  we have the chain of conclusions:
\begin{enumerate} 
    \item $\eta_l$ becomes very small with small $\tau$, leading to minimal changes in $x_{i,k+1}^t - x_{i, k}^t$;

    \item consequently, $\Delta_i^t$ is negligible;
    \item thus, $\Delta_t$ is sufficiently small;
    \label{line:3}
     
    \item given $\beta_1 = 0$, $m_t$ is extremely small;
\label{line:4}
    
    \item also, $v_t \sim 1$ because of point \ref{line:3};
    \item as a corollary, $x_{t+1} = x_t + \frac{m_t}{\sqrt{v_t} + \tau} \approx x_t$, because of \ref{line:4}.
    
\end{enumerate}   

Then, as a direct consequence, the smaller the $\tau$, the more iterations are needed to converge, since the changes in iterations are becoming smaller during to reduction of $\tau$. 
This issue fundamentally arises from neglecting $v_{-1}$ in the final analysis. 
The claim that $\sqrt{v_{t-1, j}} \leq \eta_l KG\sqrt{T}$ (page 16) doesn't incorporate $v_{-1}$, affecting the final convergence estimate. 
If $v_{-1} \sim \tau^2$, then $\frac{\Delta_t}{\sqrt{v_t} + \tau} \sim \text{const}$, resolving the issue of minor changes in $x_t$ across iterations.

\section{Experiments}
In this section, we describe the experimental setups and present the results.
\label{exper}
\subsection{Setup}
$ \ $\\
\noindent
\hspace{0.55cm}\textbf{Datasets.} 
We utilize the CIFAR-10 dataset \cite{krizhevsky2009learning} in our experiments. We chose the number of clients $M$ equal to 10 (the same number as the number of classes in the CIFAR-10 dataset). We divide the data into training and test parts in a percentage ratio: 90\%-10\%. We divide the training sample among the devices in equal number. To realize the heterogeneity of the data for each of the clients we select a "main" class of 10. We choose 30\%, 50\%, or 70\% of the "main" class for the corresponding client and add the rest data evenly from the remaining samples.
\vskip5pt

\noindent
\textbf{Metric.} Since we solve the classification problem, we use standard metrics such as cross-entropy loss and accuracy.
\vskip5pt

\noindent
\textbf{Models.} We choose the ResNet18 model \cite{he2016deep} for our analysis.
\vskip5pt

\noindent
\textbf{Optimization methods.} 
For our experiments, we implemented three different preconditioning matrices: the identity matrix (representing pure Local SGD with momentum), the matrix from Adam \cite{kingma2014adam}, and the matrix from OASIS \cite{jahani2021doubly}. In the case of using Adam and OASIS, we study two ways in which the updating of the scaling matrix works: global (as done in Algorithm \ref{alg:local_sgd_with_preconditioner}, where all devices have the same matrix and update it at the time of synchronization) and local (where each device updates its own scaling matrix at each iteration -- for this approach we do not give theoretical studies). 

For all methods, we chose a heavy-ball momentum $\beta_1$ equal to $0.9$, a scaling momentum $\beta_2$ to $0.999$, a batch size to be $256$, and a number of local iterations between communications as $18$ ($1$ epoch).

\subsection{Results}
The outcomes of the experiments are shown in Figure \ref{fig:adam}. The results, contrary to theoretical expectations, demonstrated that methods with scaling achieved the required accuracy faster than those without it.
This outcome is both classical and consequential, as the theoretical framework typically considers a general type of preconditioning matrix, which does not incorporate the specifics of adaptive scaling. Moreover, the local scaling (for which we do not provide a theory due to the fact that it is a more complex case compared to global scaling) works better for Adam than the approach from Algorithm \ref{alg:local_sgd_with_preconditioner}, but for OASIS the global scaling is no worse and sometimes even better.

\begin{figure}
\begin{minipage}{0.5\textwidth}
  \centering
\includegraphics[width =  \textwidth ]{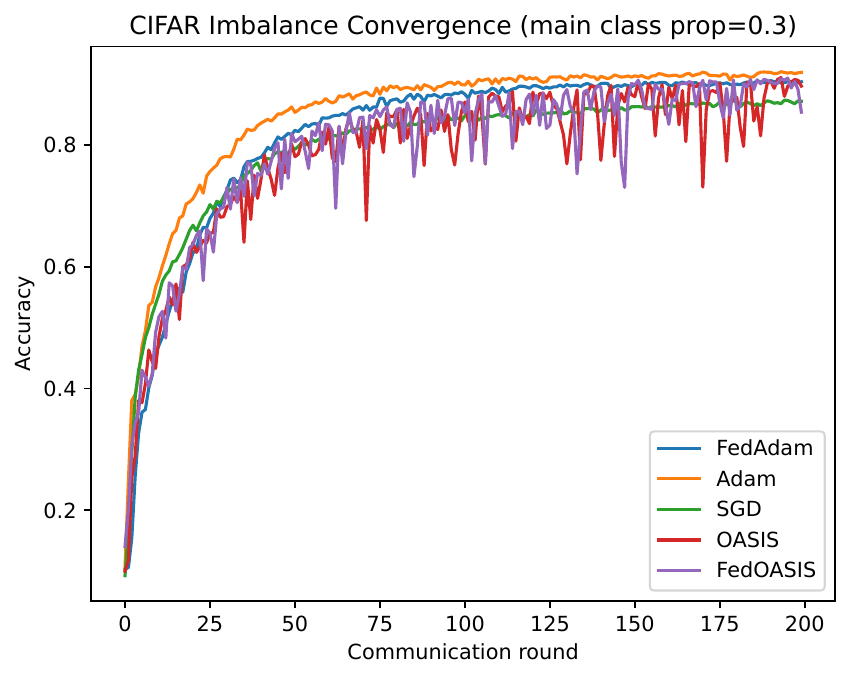}
\end{minipage}%
\begin{minipage}{0.5\textwidth}
  \centering
\includegraphics[width =  \textwidth ]{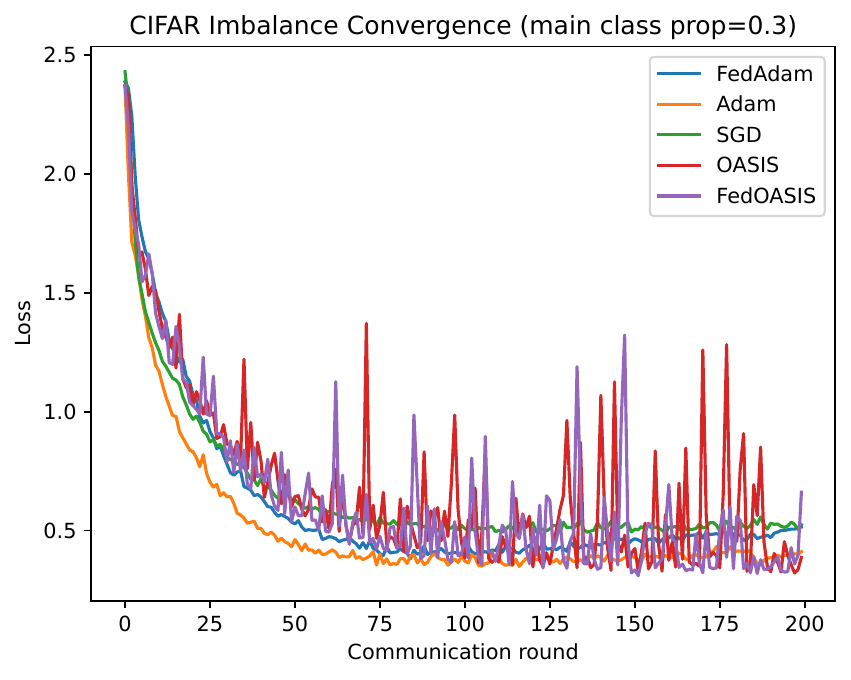}
\end{minipage}%
\\
\begin{minipage}{0.5\textwidth}
  \centering
\includegraphics[width =  \textwidth ]{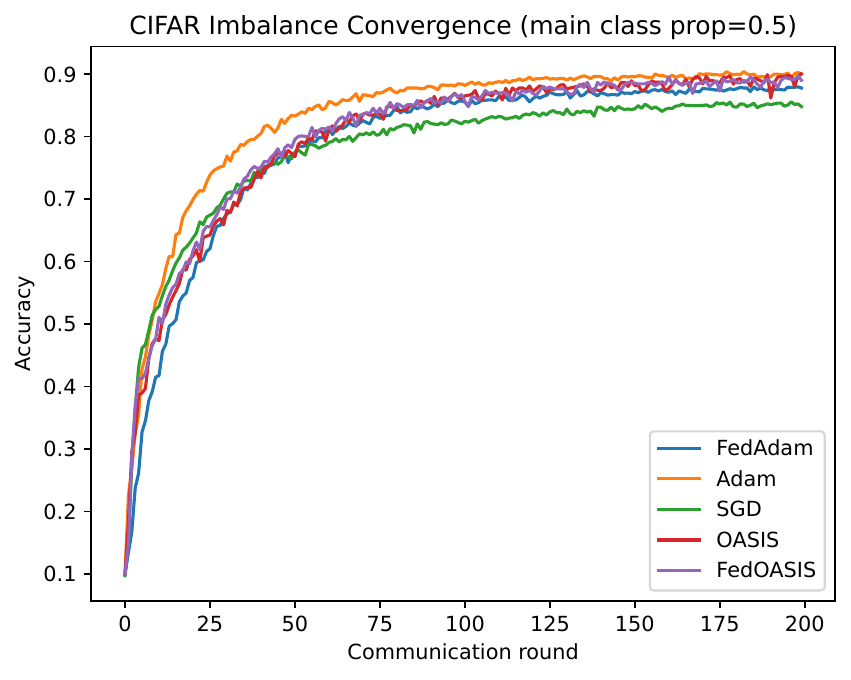}
\end{minipage}%
\begin{minipage}{0.5\textwidth}
  \centering
\includegraphics[width =  \textwidth ]{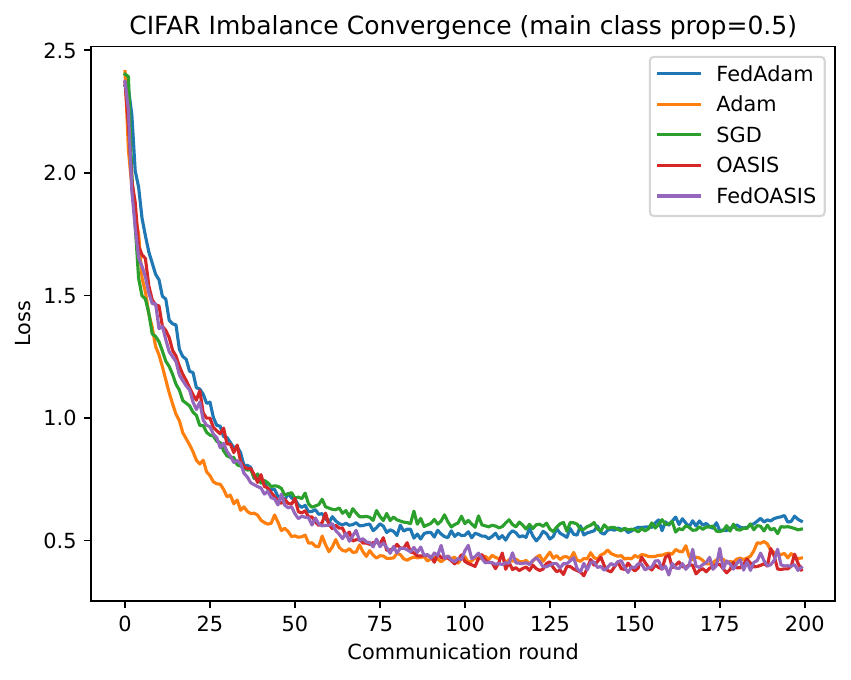}
\end{minipage}%
\\
\begin{minipage}{0.5\textwidth}
  \centering
\includegraphics[width =  \textwidth ]{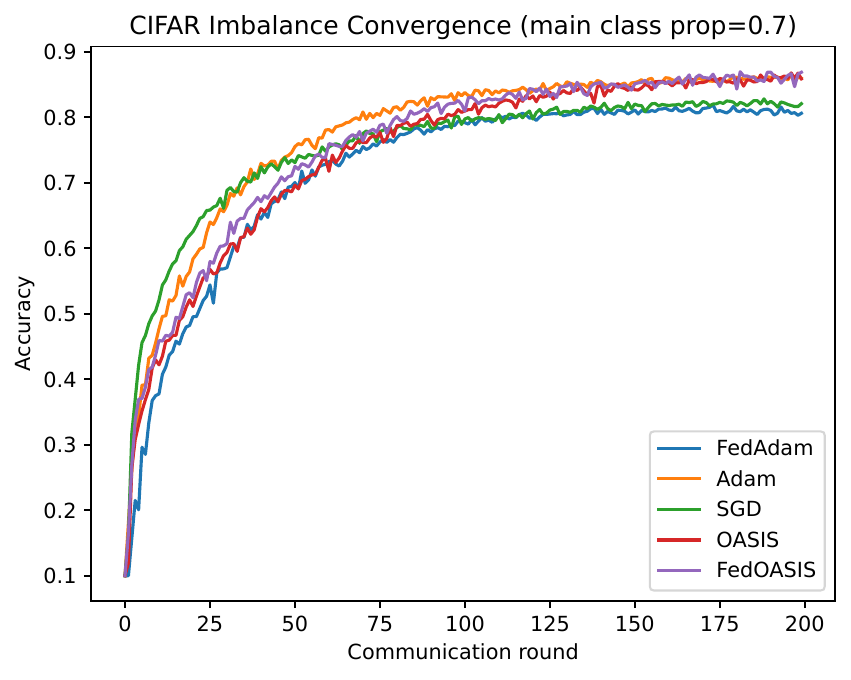}
\end{minipage}%
\begin{minipage}{0.5\textwidth}
  \centering
\includegraphics[width =  \textwidth ]{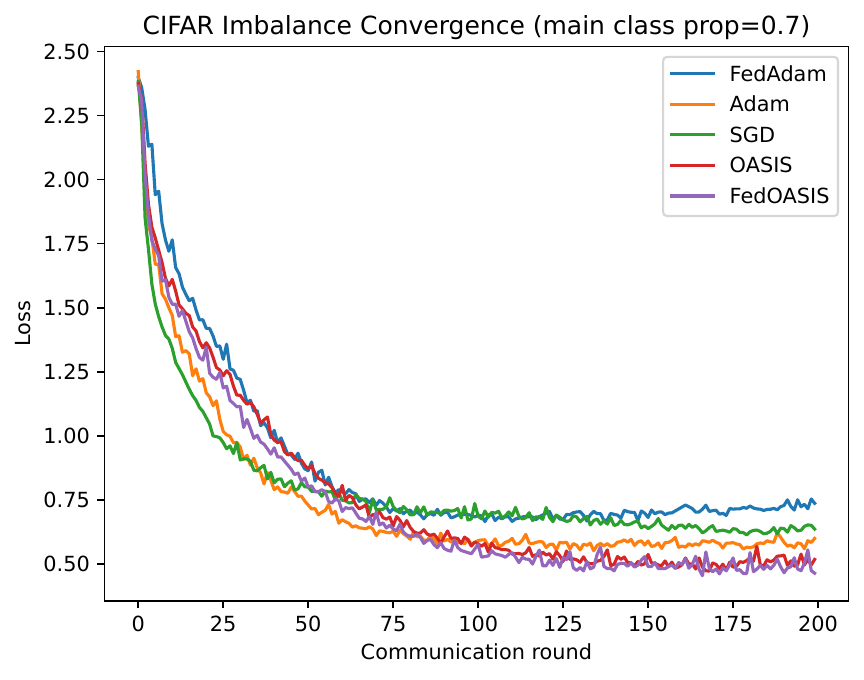}
\end{minipage}%
\caption{Comparison of different federated optimization methods: without scaling (\texttt{SGD}), Adam (local \texttt{Adam local} and global \texttt{Adam global} scalings) and OASIS (local \texttt{OASIS local} and global \texttt{OASIS global} scalings). We consider the behaviour of the accuracy (left) and loss function (right) on the ResNet 18 learning problem on CIFAR-10 with different degrees of heterogeneity: 30\% (top), 50\% (middle), 70\% (bottom) of the main class. The horizontal axis defers the synchronization/communication round number.}
\label{fig:adam}
\end{figure}

\section{Conclusion}

In this paper, we present a unified convergence analysis of a method that combines Local SGD with a preconditioning technique. We demonstrate that the theoretical convergence rate of the method is preserved, except for the introduction of a multiplicative factor, $\frac{\Gamma}{\alpha}$. 
This modification is due to our consideration of a general form for the scaling matrix. 
Additionally, we present experiments, showing that Local SGD with scaling outperforms the version without it. 
Our paper also identifies areas for future work, suggesting that one could consider specific types of preconditioning matrices to demonstrate theoretical improvements in convergence. Also an interesting question for future research is the construction of a theory of local individual scaling, which in experiments surpassed global scaling from Algorithm \ref{alg:local_sgd_with_preconditioner}.

\subsection*{Acknowledgements}

The work was done in the Laboratory of Federated Learning Problems of the ISP RAS (Supported by Grant App. No. 2 to Agreement No. 075-03-2024-214).
The work was partially conducted while S. Chezhegov was a visiting research assistant in Mohamed bin Zayed University of Artificial Intelligence (MBZUAI).

\bibliography{bibliography}


\begin{thebibliography}{49}
\ifx \bisbn   \undefined \def \bisbn  #1{ISBN #1}\fi
\ifx \binits  \undefined \def \binits#1{#1}\fi
\ifx \bauthor  \undefined \def \bauthor#1{#1}\fi
\ifx \batitle  \undefined \def \batitle#1{#1}\fi
\ifx \bjtitle  \undefined \def \bjtitle#1{#1}\fi
\ifx \bvolume  \undefined \def \bvolume#1{\textbf{#1}}\fi
\ifx \byear  \undefined \def \byear#1{#1}\fi
\ifx \bissue  \undefined \def \bissue#1{#1}\fi
\ifx \bfpage  \undefined \def \bfpage#1{#1}\fi
\ifx \blpage  \undefined \def \blpage #1{#1}\fi
\ifx \burl  \undefined \def \burl#1{\textsf{#1}}\fi
\ifx \doiurl  \undefined \def \doiurl#1{\url{https://doi.org/#1}}\fi
\ifx \betal  \undefined \def \betal{\textit{et al.}}\fi
\ifx \binstitute  \undefined \def \binstitute#1{#1}\fi
\ifx \binstitutionaled  \undefined \def \binstitutionaled#1{#1}\fi
\ifx \bctitle  \undefined \def \bctitle#1{#1}\fi
\ifx \beditor  \undefined \def \beditor#1{#1}\fi
\ifx \bpublisher  \undefined \def \bpublisher#1{#1}\fi
\ifx \bbtitle  \undefined \def \bbtitle#1{#1}\fi
\ifx \bedition  \undefined \def \bedition#1{#1}\fi
\ifx \bseriesno  \undefined \def \bseriesno#1{#1}\fi
\ifx \blocation  \undefined \def \blocation#1{#1}\fi
\ifx \bsertitle  \undefined \def \bsertitle#1{#1}\fi
\ifx \bsnm \undefined \def \bsnm#1{#1}\fi
\ifx \bsuffix \undefined \def \bsuffix#1{#1}\fi
\ifx \bparticle \undefined \def \bparticle#1{#1}\fi
\ifx \barticle \undefined \def \barticle#1{#1}\fi
\bibcommenthead
\ifx \bconfdate \undefined \def \bconfdate #1{#1}\fi
\ifx \botherref \undefined \def \botherref #1{#1}\fi
\ifx \url \undefined \def \url#1{\textsf{#1}}\fi
\ifx \bchapter \undefined \def \bchapter#1{#1}\fi
\ifx \bbook \undefined \def \bbook#1{#1}\fi
\ifx \bcomment \undefined \def \bcomment#1{#1}\fi
\ifx \oauthor \undefined \def \oauthor#1{#1}\fi
\ifx \citeauthoryear \undefined \def \citeauthoryear#1{#1}\fi
\ifx \endbibitem  \undefined \def \endbibitem {}\fi
\ifx \bconflocation  \undefined \def \bconflocation#1{#1}\fi
\ifx \arxivurl  \undefined \def \arxivurl#1{\textsf{#1}}\fi
\csname PreBibitemsHook\endcsname

\bibitem[\protect\citeauthoryear{Shalev-Shwartz and Ben-David}{2014}]{shalev2014understanding}
\begin{botherref}
\oauthor{\bsnm{Shalev-Shwartz}, \binits{S.}},
\oauthor{\bsnm{Ben-David}, \binits{S.}}:
Understanding machine learning: From theory to algorithms
(2014)
\end{botherref}
\endbibitem

\bibitem[\protect\citeauthoryear{Goodfellow et~al.}{2016}]{goodfellow2016deep}
\begin{botherref}
\oauthor{\bsnm{Goodfellow}, \binits{I.}},
\oauthor{\bsnm{Bengio}, \binits{Y.}},
\oauthor{\bsnm{Courville}, \binits{A.}}:
Deep learning
(2016)
\end{botherref}
\endbibitem

\bibitem[\protect\citeauthoryear{Arora et~al.}{2018}]{arora2018optimization}
\begin{bchapter}
\bauthor{\bsnm{Arora}, \binits{S.}},
\bauthor{\bsnm{Cohen}, \binits{N.}},
\bauthor{\bsnm{Hazan}, \binits{E.}}:
\bctitle{On the optimization of deep networks: Implicit acceleration by overparameterization}.
In: \bbtitle{International Conference on Machine Learning},
pp. \bfpage{244}--\blpage{253}
(\byear{2018}).
\bcomment{PMLR}
\end{bchapter}
\endbibitem

\bibitem[\protect\citeauthoryear{Smith et~al.}{2018}]{smith2018cocoa}
\begin{barticle}
\bauthor{\bsnm{Smith}, \binits{V.}},
\bauthor{\bsnm{Forte}, \binits{S.}},
\bauthor{\bsnm{Ma}, \binits{C.}},
\bauthor{\bsnm{Tak{\'a}{\v{c}}}, \binits{M.}},
\bauthor{\bsnm{Jordan}, \binits{M.I.}},
\bauthor{\bsnm{Jaggi}, \binits{M.}}:
\batitle{Cocoa: A general framework for communication-efficient distributed optimization}.
\bjtitle{Journal of Machine Learning Research}
\bvolume{18}(\bissue{230}),
\bfpage{1}--\blpage{49}
(\byear{2018})
\end{barticle}
\endbibitem

\bibitem[\protect\citeauthoryear{Mishchenko et~al.}{2019}]{mishchenko2019distributed}
\begin{botherref}
\oauthor{\bsnm{Mishchenko}, \binits{K.}},
\oauthor{\bsnm{Gorbunov}, \binits{E.}},
\oauthor{\bsnm{Tak{\'a}{\v{c}}}, \binits{M.}},
\oauthor{\bsnm{Richt{\'a}rik}, \binits{P.}}:
Distributed learning with compressed gradient differences.
arXiv preprint arXiv:1901.09269
(2019)
\end{botherref}
\endbibitem

\bibitem[\protect\citeauthoryear{Verbraeken et~al.}{2020}]{verbraeken2020survey}
\begin{barticle}
\bauthor{\bsnm{Verbraeken}, \binits{J.}},
\bauthor{\bsnm{Wolting}, \binits{M.}},
\bauthor{\bsnm{Katzy}, \binits{J.}},
\bauthor{\bsnm{Kloppenburg}, \binits{J.}},
\bauthor{\bsnm{Verbelen}, \binits{T.}},
\bauthor{\bsnm{Rellermeyer}, \binits{J.S.}}:
\batitle{A survey on distributed machine learning}.
\bjtitle{Acm computing surveys (csur)}
\bvolume{53}(\bissue{2}),
\bfpage{1}--\blpage{33}
(\byear{2020})
\end{barticle}
\endbibitem

\bibitem[\protect\citeauthoryear{Chraibi et~al.}{2019}]{chraibi2019distributed}
\begin{botherref}
\oauthor{\bsnm{Chraibi}, \binits{S.}},
\oauthor{\bsnm{Khaled}, \binits{A.}},
\oauthor{\bsnm{Kovalev}, \binits{D.}},
\oauthor{\bsnm{Richt{\'a}rik}, \binits{P.}},
\oauthor{\bsnm{Salim}, \binits{A.}},
\oauthor{\bsnm{Tak{\'a}{\v{c}}}, \binits{M.}}:
Distributed fixed point methods with compressed iterates.
arXiv preprint arXiv:1912.09925
(2019)
\end{botherref}
\endbibitem

\bibitem[\protect\citeauthoryear{Pirau et~al.}{2024}]{pirau2024preconditioning}
\begin{barticle}
\bauthor{\bsnm{Pirau}, \binits{V.}},
\bauthor{\bsnm{Beznosikov}, \binits{A.}},
\bauthor{\bsnm{Tak{\'a}{\v{c}}}, \binits{M.}},
\bauthor{\bsnm{Matyukhin}, \binits{V.}},
\bauthor{\bsnm{Gasnikov}, \binits{A.}}:
\batitle{Preconditioning meets biased compression for efficient distributed optimization}.
\bjtitle{Computational Management Science}
\bvolume{21}(\bissue{1}),
\bfpage{14}
(\byear{2024})
\end{barticle}
\endbibitem

\bibitem[\protect\citeauthoryear{Kone{\v{c}}n{\`y} et~al.}{2016}]{konevcny2016federated}
\begin{botherref}
\oauthor{\bsnm{Kone{\v{c}}n{\`y}}, \binits{J.}},
\oauthor{\bsnm{McMahan}, \binits{H.B.}},
\oauthor{\bsnm{Yu}, \binits{F.X.}},
\oauthor{\bsnm{Richt{\'a}rik}, \binits{P.}},
\oauthor{\bsnm{Suresh}, \binits{A.T.}},
\oauthor{\bsnm{Bacon}, \binits{D.}}:
Federated learning: Strategies for improving communication efficiency.
arXiv preprint arXiv:1610.05492
(2016)
\end{botherref}
\endbibitem

\bibitem[\protect\citeauthoryear{Kairouz et~al.}{2021}]{kairouz2021advances}
\begin{barticle}
\bauthor{\bsnm{Kairouz}, \binits{P.}},
\bauthor{\bsnm{McMahan}, \binits{H.B.}},
\bauthor{\bsnm{Avent}, \binits{B.}},
\bauthor{\bsnm{Bellet}, \binits{A.}},
\bauthor{\bsnm{Bennis}, \binits{M.}},
\bauthor{\bsnm{Bhagoji}, \binits{A.N.}},
\bauthor{\bsnm{Bonawitz}, \binits{K.}},
\bauthor{\bsnm{Charles}, \binits{Z.}},
\bauthor{\bsnm{Cormode}, \binits{G.}},
\bauthor{\bsnm{Cummings}, \binits{R.}}, \betal:
\batitle{Advances and open problems in federated learning}.
\bjtitle{Foundations and Trends{\textregistered} in Machine Learning}
\bvolume{14}(\bissue{1--2}),
\bfpage{1}--\blpage{210}
(\byear{2021})
\end{barticle}
\endbibitem

\bibitem[\protect\citeauthoryear{Karimireddy et~al.}{2020}]{karimireddy2020scaffold}
\begin{bchapter}
\bauthor{\bsnm{Karimireddy}, \binits{S.P.}},
\bauthor{\bsnm{Kale}, \binits{S.}},
\bauthor{\bsnm{Mohri}, \binits{M.}},
\bauthor{\bsnm{Reddi}, \binits{S.}},
\bauthor{\bsnm{Stich}, \binits{S.}},
\bauthor{\bsnm{Suresh}, \binits{A.T.}}:
\bctitle{Scaffold: Stochastic controlled averaging for federated learning}.
In: \bbtitle{International Conference on Machine Learning},
pp. \bfpage{5132}--\blpage{5143}
(\byear{2020}).
\bcomment{PMLR}
\end{bchapter}
\endbibitem

\bibitem[\protect\citeauthoryear{Arjevani and Shamir}{2015}]{arjevani2015communication}
\begin{botherref}
\oauthor{\bsnm{Arjevani}, \binits{Y.}},
\oauthor{\bsnm{Shamir}, \binits{O.}}:
Communication complexity of distributed convex learning and optimization.
Advances in neural information processing systems
\textbf{28}
(2015)
\end{botherref}
\endbibitem

\bibitem[\protect\citeauthoryear{Alistarh et~al.}{2017}]{alistarh2017qsgd}
\begin{botherref}
\oauthor{\bsnm{Alistarh}, \binits{D.}},
\oauthor{\bsnm{Grubic}, \binits{D.}},
\oauthor{\bsnm{Li}, \binits{J.}},
\oauthor{\bsnm{Tomioka}, \binits{R.}},
\oauthor{\bsnm{Vojnovic}, \binits{M.}}:
Qsgd: Communication-efficient sgd via gradient quantization and encoding.
Advances in neural information processing systems
\textbf{30}
(2017)
\end{botherref}
\endbibitem

\bibitem[\protect\citeauthoryear{Beznosikov et~al.}{2020}]{beznosikov2020biased}
\begin{botherref}
\oauthor{\bsnm{Beznosikov}, \binits{A.}},
\oauthor{\bsnm{Horv{\'a}th}, \binits{S.}},
\oauthor{\bsnm{Richt{\'a}rik}, \binits{P.}},
\oauthor{\bsnm{Safaryan}, \binits{M.}}:
On biased compression for distributed learning.
arXiv preprint arXiv:2002.12410
(2020)
\end{botherref}
\endbibitem

\bibitem[\protect\citeauthoryear{Mangasarian}{1995}]{mangasarian1995parallel}
\begin{barticle}
\bauthor{\bsnm{Mangasarian}, \binits{L.}}:
\batitle{Parallel gradient distribution in unconstrained optimization}.
\bjtitle{SIAM Journal on Control and Optimization}
\bvolume{33}(\bissue{6}),
\bfpage{1916}--\blpage{1925}
(\byear{1995})
\end{barticle}
\endbibitem

\bibitem[\protect\citeauthoryear{Stich}{2018}]{stich2018local}
\begin{botherref}
\oauthor{\bsnm{Stich}, \binits{S.U.}}:
Local sgd converges fast and communicates little.
arXiv preprint arXiv:1805.09767
(2018)
\end{botherref}
\endbibitem

\bibitem[\protect\citeauthoryear{McMahan et~al.}{2017}]{mcmahan2017communication}
\begin{bchapter}
\bauthor{\bsnm{McMahan}, \binits{B.}},
\bauthor{\bsnm{Moore}, \binits{E.}},
\bauthor{\bsnm{Ramage}, \binits{D.}},
\bauthor{\bsnm{Hampson}, \binits{S.}},
\bauthor{\bsnm{Arcas}, \binits{B.A.}}:
\bctitle{Communication-efficient learning of deep networks from decentralized data}.
In: \bbtitle{Artificial Intelligence and Statistics},
pp. \bfpage{1273}--\blpage{1282}
(\byear{2017}).
\bcomment{PMLR}
\end{bchapter}
\endbibitem

\bibitem[\protect\citeauthoryear{Wang et~al.}{2019}]{wang2019slowmo}
\begin{botherref}
\oauthor{\bsnm{Wang}, \binits{J.}},
\oauthor{\bsnm{Tantia}, \binits{V.}},
\oauthor{\bsnm{Ballas}, \binits{N.}},
\oauthor{\bsnm{Rabbat}, \binits{M.}}:
Slowmo: Improving communication-efficient distributed sgd with slow momentum.
arXiv preprint arXiv:1910.00643
(2019)
\end{botherref}
\endbibitem

\bibitem[\protect\citeauthoryear{Basu et~al.}{2020}]{basu2020qsparse}
\begin{barticle}
\bauthor{\bsnm{Basu}, \binits{D.}},
\bauthor{\bsnm{Data}, \binits{D.}},
\bauthor{\bsnm{Karakus}, \binits{C.}},
\bauthor{\bsnm{Diggavi}, \binits{S.N.}}:
\batitle{Qsparse-local-sgd: Distributed sgd with quantization, sparsification, and local computations}.
\bjtitle{IEEE Journal on Selected Areas in Information Theory}
\bvolume{1}(\bissue{1}),
\bfpage{217}--\blpage{226}
(\byear{2020})
\end{barticle}
\endbibitem

\bibitem[\protect\citeauthoryear{Reisizadeh et~al.}{2020}]{reisizadeh2020fedpaq}
\begin{bchapter}
\bauthor{\bsnm{Reisizadeh}, \binits{A.}},
\bauthor{\bsnm{Mokhtari}, \binits{A.}},
\bauthor{\bsnm{Hassani}, \binits{H.}},
\bauthor{\bsnm{Jadbabaie}, \binits{A.}},
\bauthor{\bsnm{Pedarsani}, \binits{R.}}:
\bctitle{Fedpaq: A communication-efficient federated learning method with periodic averaging and quantization}.
In: \bbtitle{International Conference on Artificial Intelligence and Statistics},
pp. \bfpage{2021}--\blpage{2031}
(\byear{2020}).
\bcomment{PMLR}
\end{bchapter}
\endbibitem

\bibitem[\protect\citeauthoryear{Liang et~al.}{2019}]{liang2019variance}
\begin{botherref}
\oauthor{\bsnm{Liang}, \binits{X.}},
\oauthor{\bsnm{Shen}, \binits{S.}},
\oauthor{\bsnm{Liu}, \binits{J.}},
\oauthor{\bsnm{Pan}, \binits{Z.}},
\oauthor{\bsnm{Chen}, \binits{E.}},
\oauthor{\bsnm{Cheng}, \binits{Y.}}:
Variance reduced local sgd with lower communication complexity.
arXiv preprint arXiv:1912.12844
(2019)
\end{botherref}
\endbibitem

\bibitem[\protect\citeauthoryear{Sharma et~al.}{2019}]{sharma2019parallel}
\begin{botherref}
\oauthor{\bsnm{Sharma}, \binits{P.}},
\oauthor{\bsnm{Kafle}, \binits{S.}},
\oauthor{\bsnm{Khanduri}, \binits{P.}},
\oauthor{\bsnm{Bulusu}, \binits{S.}},
\oauthor{\bsnm{Rajawat}, \binits{K.}},
\oauthor{\bsnm{Varshney}, \binits{P.K.}}:
Parallel restarted spider--communication efficient distributed nonconvex optimization with optimal computation complexity.
arXiv preprint arXiv:1912.06036
(2019)
\end{botherref}
\endbibitem

\bibitem[\protect\citeauthoryear{Mishchenko et~al.}{2022}]{mishchenko2022proxskip}
\begin{bchapter}
\bauthor{\bsnm{Mishchenko}, \binits{K.}},
\bauthor{\bsnm{Malinovsky}, \binits{G.}},
\bauthor{\bsnm{Stich}, \binits{S.}},
\bauthor{\bsnm{Richt{\'a}rik}, \binits{P.}}:
\bctitle{Proxskip: Yes! local gradient steps provably lead to communication acceleration! finally!}
In: \bbtitle{International Conference on Machine Learning},
pp. \bfpage{15750}--\blpage{15769}
(\byear{2022}).
\bcomment{PMLR}
\end{bchapter}
\endbibitem

\bibitem[\protect\citeauthoryear{Shamir et~al.}{2014}]{shamir2014communication}
\begin{bchapter}
\bauthor{\bsnm{Shamir}, \binits{O.}},
\bauthor{\bsnm{Srebro}, \binits{N.}},
\bauthor{\bsnm{Zhang}, \binits{T.}}:
\bctitle{Communication-efficient distributed optimization using an approximate newton-type method}.
In: \bbtitle{International Conference on Machine Learning},
pp. \bfpage{1000}--\blpage{1008}
(\byear{2014}).
\bcomment{PMLR}
\end{bchapter}
\endbibitem

\bibitem[\protect\citeauthoryear{Hendrikx et~al.}{2020}]{hendrikx2020statistically}
\begin{bchapter}
\bauthor{\bsnm{Hendrikx}, \binits{H.}},
\bauthor{\bsnm{Xiao}, \binits{L.}},
\bauthor{\bsnm{Bubeck}, \binits{S.}},
\bauthor{\bsnm{Bach}, \binits{F.}},
\bauthor{\bsnm{Massoulie}, \binits{L.}}:
\bctitle{Statistically preconditioned accelerated gradient method for distributed optimization}.
In: \bbtitle{International Conference on Machine Learning},
pp. \bfpage{4203}--\blpage{4227}
(\byear{2020}).
\bcomment{PMLR}
\end{bchapter}
\endbibitem

\bibitem[\protect\citeauthoryear{Beznosikov et~al.}{2021}]{beznosikov2021distributed}
\begin{barticle}
\bauthor{\bsnm{Beznosikov}, \binits{A.}},
\bauthor{\bsnm{Scutari}, \binits{G.}},
\bauthor{\bsnm{Rogozin}, \binits{A.}},
\bauthor{\bsnm{Gasnikov}, \binits{A.}}:
\batitle{Distributed saddle-point problems under data similarity}.
\bjtitle{Advances in Neural Information Processing Systems}
\bvolume{34},
\bfpage{8172}--\blpage{8184}
(\byear{2021})
\end{barticle}
\endbibitem

\bibitem[\protect\citeauthoryear{Kovalev et~al.}{2022}]{kovalev2022optimal}
\begin{botherref}
\oauthor{\bsnm{Kovalev}, \binits{D.}},
\oauthor{\bsnm{Beznosikov}, \binits{A.}},
\oauthor{\bsnm{Borodich}, \binits{E.}},
\oauthor{\bsnm{Gasnikov}, \binits{A.}},
\oauthor{\bsnm{Scutari}, \binits{G.}}:
Optimal gradient sliding and its application to distributed optimization under similarity.
arXiv preprint arXiv:2205.15136
(2022)
\end{botherref}
\endbibitem

\bibitem[\protect\citeauthoryear{Kingma and Ba}{2014}]{kingma2014adam}
\begin{botherref}
\oauthor{\bsnm{Kingma}, \binits{D.P.}},
\oauthor{\bsnm{Ba}, \binits{J.}}:
Adam: A method for stochastic optimization.
arXiv preprint arXiv:1412.6980
(2014)
\end{botherref}
\endbibitem

\bibitem[\protect\citeauthoryear{Tieleman et~al.}{2012}]{tieleman2012lecture}
\begin{barticle}
\bauthor{\bsnm{Tieleman}, \binits{T.}},
\bauthor{\bsnm{Hinton}, \binits{G.}}, \betal:
\batitle{Lecture 6.5-rmsprop: Divide the gradient by a running average of its recent magnitude}.
\bjtitle{COURSERA: Neural networks for machine learning}
\bvolume{4}(\bissue{2}),
\bfpage{26}--\blpage{31}
(\byear{2012})
\end{barticle}
\endbibitem

\bibitem[\protect\citeauthoryear{Duchi et~al.}{2011}]{duchi2011adaptive}
\begin{botherref}
\oauthor{\bsnm{Duchi}, \binits{J.}},
\oauthor{\bsnm{Hazan}, \binits{E.}},
\oauthor{\bsnm{Singer}, \binits{Y.}}:
Adaptive subgradient methods for online learning and stochastic optimization.
Journal of machine learning research
\textbf{12}(7)
(2011)
\end{botherref}
\endbibitem

\bibitem[\protect\citeauthoryear{Bekas et~al.}{2007}]{bekas2007estimator}
\begin{barticle}
\bauthor{\bsnm{Bekas}, \binits{C.}},
\bauthor{\bsnm{Kokiopoulou}, \binits{E.}},
\bauthor{\bsnm{Saad}, \binits{Y.}}:
\batitle{An estimator for the diagonal of a matrix}.
\bjtitle{Applied numerical mathematics}
\bvolume{57}(\bissue{11-12}),
\bfpage{1214}--\blpage{1229}
(\byear{2007})
\end{barticle}
\endbibitem

\bibitem[\protect\citeauthoryear{Sadiev et~al.}{2024}]{sadiev2024stochastic}
\begin{botherref}
\oauthor{\bsnm{Sadiev}, \binits{A.}},
\oauthor{\bsnm{Beznosikov}, \binits{A.}},
\oauthor{\bsnm{Almansoori}, \binits{A.J.}},
\oauthor{\bsnm{Kamzolov}, \binits{D.}},
\oauthor{\bsnm{Tappenden}, \binits{R.}},
\oauthor{\bsnm{Tak{\'a}{\v{c}}}, \binits{M.}}:
Stochastic gradient methods with preconditioned updates.
Journal of Optimization Theory and Applications,
1--19
(2024)
\end{botherref}
\endbibitem

\bibitem[\protect\citeauthoryear{Jahani et~al.}{2021}]{jahani2021doubly}
\begin{botherref}
\oauthor{\bsnm{Jahani}, \binits{M.}},
\oauthor{\bsnm{Rusakov}, \binits{S.}},
\oauthor{\bsnm{Shi}, \binits{Z.}},
\oauthor{\bsnm{Richt{\'a}rik}, \binits{P.}},
\oauthor{\bsnm{Mahoney}, \binits{M.W.}},
\oauthor{\bsnm{Tak{\'a}{\v{c}}}, \binits{M.}}:
Doubly adaptive scaled algorithm for machine learning using second-order information.
arXiv preprint arXiv:2109.05198
(2021)
\end{botherref}
\endbibitem

\bibitem[\protect\citeauthoryear{Dennis and Mor{\'e}}{1977}]{dennis1977quasi}
\begin{barticle}
\bauthor{\bsnm{Dennis}, \binits{J.E.} \bsuffix{Jr}},
\bauthor{\bsnm{Mor{\'e}}, \binits{J.J.}}:
\batitle{Quasi-newton methods, motivation and theory}.
\bjtitle{SIAM review}
\bvolume{19}(\bissue{1}),
\bfpage{46}--\blpage{89}
(\byear{1977})
\end{barticle}
\endbibitem

\bibitem[\protect\citeauthoryear{Fletcher}{2013}]{fletcher2013practical}
\begin{botherref}
\oauthor{\bsnm{Fletcher}, \binits{R.}}:
Practical methods of optimization
(2013)
\end{botherref}
\endbibitem

\bibitem[\protect\citeauthoryear{Khaled et~al.}{2020}]{khaled2020tighter}
\begin{bchapter}
\bauthor{\bsnm{Khaled}, \binits{A.}},
\bauthor{\bsnm{Mishchenko}, \binits{K.}},
\bauthor{\bsnm{Richt{\'a}rik}, \binits{P.}}:
\bctitle{Tighter theory for local sgd on identical and heterogeneous data}.
In: \bbtitle{International Conference on Artificial Intelligence and Statistics},
pp. \bfpage{4519}--\blpage{4529}
(\byear{2020}).
\bcomment{PMLR}
\end{bchapter}
\endbibitem

\bibitem[\protect\citeauthoryear{Koloskova et~al.}{2020}]{koloskova2020unified}
\begin{bchapter}
\bauthor{\bsnm{Koloskova}, \binits{A.}},
\bauthor{\bsnm{Loizou}, \binits{N.}},
\bauthor{\bsnm{Boreiri}, \binits{S.}},
\bauthor{\bsnm{Jaggi}, \binits{M.}},
\bauthor{\bsnm{Stich}, \binits{S.}}:
\bctitle{A unified theory of decentralized sgd with changing topology and local updates}.
In: \bbtitle{International Conference on Machine Learning},
pp. \bfpage{5381}--\blpage{5393}
(\byear{2020}).
\bcomment{PMLR}
\end{bchapter}
\endbibitem

\bibitem[\protect\citeauthoryear{Beznosikov et~al.}{2022}]{beznosikov2022decentralized}
\begin{barticle}
\bauthor{\bsnm{Beznosikov}, \binits{A.}},
\bauthor{\bsnm{Dvurechenskii}, \binits{P.}},
\bauthor{\bsnm{Koloskova}, \binits{A.}},
\bauthor{\bsnm{Samokhin}, \binits{V.}},
\bauthor{\bsnm{Stich}, \binits{S.U.}},
\bauthor{\bsnm{Gasnikov}, \binits{A.}}:
\batitle{Decentralized local stochastic extra-gradient for variational inequalities}.
\bjtitle{Advances in Neural Information Processing Systems}
\bvolume{35},
\bfpage{38116}--\blpage{38133}
(\byear{2022})
\end{barticle}
\endbibitem

\bibitem[\protect\citeauthoryear{Glasgow et~al.}{2022}]{glasgow2022sharp}
\begin{bchapter}
\bauthor{\bsnm{Glasgow}, \binits{M.R.}},
\bauthor{\bsnm{Yuan}, \binits{H.}},
\bauthor{\bsnm{Ma}, \binits{T.}}:
\bctitle{Sharp bounds for federated averaging (local sgd) and continuous perspective}.
In: \bbtitle{International Conference on Artificial Intelligence and Statistics},
pp. \bfpage{9050}--\blpage{9090}
(\byear{2022}).
\bcomment{PMLR}
\end{bchapter}
\endbibitem

\bibitem[\protect\citeauthoryear{Sadiev et~al.}{2022}]{sadiev2022stochastic}
\begin{botherref}
\oauthor{\bsnm{Sadiev}, \binits{A.}},
\oauthor{\bsnm{Beznosikov}, \binits{A.}},
\oauthor{\bsnm{Almansoori}, \binits{A.J.}},
\oauthor{\bsnm{Kamzolov}, \binits{D.}},
\oauthor{\bsnm{Tappenden}, \binits{R.}},
\oauthor{\bsnm{Tak{\'a}{\v{c}}}, \binits{M.}}:
Stochastic gradient methods with preconditioned updates.
arXiv preprint arXiv:2206.00285
(2022)
\end{botherref}
\endbibitem

\bibitem[\protect\citeauthoryear{Beznosikov et~al.}{2022}]{beznosikov2022scaled}
\begin{botherref}
\oauthor{\bsnm{Beznosikov}, \binits{A.}},
\oauthor{\bsnm{Alanov}, \binits{A.}},
\oauthor{\bsnm{Kovalev}, \binits{D.}},
\oauthor{\bsnm{Tak{\'a}{\v{c}}}, \binits{M.}},
\oauthor{\bsnm{Gasnikov}, \binits{A.}}:
On scaled methods for saddle point problems.
arXiv preprint arXiv:2206.08303
(2022)
\end{botherref}
\endbibitem

\bibitem[\protect\citeauthoryear{Reddi et~al.}{2020}]{reddi2020adaptive}
\begin{botherref}
\oauthor{\bsnm{Reddi}, \binits{S.}},
\oauthor{\bsnm{Charles}, \binits{Z.}},
\oauthor{\bsnm{Zaheer}, \binits{M.}},
\oauthor{\bsnm{Garrett}, \binits{Z.}},
\oauthor{\bsnm{Rush}, \binits{K.}},
\oauthor{\bsnm{Kone{\v{c}}n{\`y}}, \binits{J.}},
\oauthor{\bsnm{Kumar}, \binits{S.}},
\oauthor{\bsnm{McMahan}, \binits{H.B.}}:
Adaptive federated optimization.
arXiv preprint arXiv:2003.00295
(2020)
\end{botherref}
\endbibitem

\bibitem[\protect\citeauthoryear{Nesterov}{2003}]{nesterov2003introductory}
\begin{botherref}
\oauthor{\bsnm{Nesterov}, \binits{Y.}}:
Introductory lectures on convex optimization: A basic course
\textbf{87}
(2003)
\end{botherref}
\endbibitem

\bibitem[\protect\citeauthoryear{Nemirovskij and Yudin}{1983}]{nemirovskij1983problem}
\begin{botherref}
\oauthor{\bsnm{Nemirovskij}, \binits{A.S.}},
\oauthor{\bsnm{Yudin}, \binits{D.B.}}:
Problem complexity and method efficiency in optimization
(1983)
\end{botherref}
\endbibitem

\bibitem[\protect\citeauthoryear{Nguyen et~al.}{2018}]{nguyen2018sgd}
\begin{bchapter}
\bauthor{\bsnm{Nguyen}, \binits{L.}},
\bauthor{\bsnm{Nguyen}, \binits{P.H.}},
\bauthor{\bsnm{Dijk}, \binits{M.}},
\bauthor{\bsnm{Richt{\'a}rik}, \binits{P.}},
\bauthor{\bsnm{Scheinberg}, \binits{K.}},
\bauthor{\bsnm{Tak{\'a}c}, \binits{M.}}:
\bctitle{Sgd and hogwild! convergence without the bounded gradients assumption}.
In: \bbtitle{International Conference on Machine Learning},
pp. \bfpage{3750}--\blpage{3758}
(\byear{2018}).
\bcomment{PMLR}
\end{bchapter}
\endbibitem

\bibitem[\protect\citeauthoryear{Yao et~al.}{2021}]{yao2021adahessian}
\begin{bchapter}
\bauthor{\bsnm{Yao}, \binits{Z.}},
\bauthor{\bsnm{Gholami}, \binits{A.}},
\bauthor{\bsnm{Shen}, \binits{S.}},
\bauthor{\bsnm{Mustafa}, \binits{M.}},
\bauthor{\bsnm{Keutzer}, \binits{K.}},
\bauthor{\bsnm{Mahoney}, \binits{M.}}:
\bctitle{Adahessian: An adaptive second order optimizer for machine learning}.
In: \bbtitle{Proceedings of the AAAI Conference on Artificial Intelligence},
vol. \bseriesno{35},
pp. \bfpage{10665}--\blpage{10673}
(\byear{2021})
\end{bchapter}
\endbibitem

\bibitem[\protect\citeauthoryear{D{\'e}fossez et~al.}{2020}]{defossez2020simple}
\begin{botherref}
\oauthor{\bsnm{D{\'e}fossez}, \binits{A.}},
\oauthor{\bsnm{Bottou}, \binits{L.}},
\oauthor{\bsnm{Bach}, \binits{F.}},
\oauthor{\bsnm{Usunier}, \binits{N.}}:
A simple convergence proof of adam and adagrad.
arXiv preprint arXiv:2003.02395
(2020)
\end{botherref}
\endbibitem

\bibitem[\protect\citeauthoryear{Krizhevsky et~al.}{2009}]{krizhevsky2009learning}
\begin{botherref}
\oauthor{\bsnm{Krizhevsky}, \binits{A.}},
\oauthor{\bsnm{Hinton}, \binits{G.}}, et al.:
Learning multiple layers of features from tiny images
(2009)
\end{botherref}
\endbibitem

\bibitem[\protect\citeauthoryear{He et~al.}{2016}]{he2016deep}
\begin{bchapter}
\bauthor{\bsnm{He}, \binits{K.}},
\bauthor{\bsnm{Zhang}, \binits{X.}},
\bauthor{\bsnm{Ren}, \binits{S.}},
\bauthor{\bsnm{Sun}, \binits{J.}}:
\bctitle{Deep residual learning for image recognition}.
In: \bbtitle{Proceedings of the IEEE Conference on Computer Vision and Pattern Recognition},
pp. \bfpage{770}--\blpage{778}
(\byear{2016})
\end{bchapter}
\endbibitem

\end{thebibliography}

\clearpage
\appendix

\section{Basic facts and auxiliary lemmas}
We use a notation similar to that of \cite{stich2018local} and denote the sequence of time stamps when synchronization happens as  $(t_{p})_{p=1}^{\infty}$. Given stochastic gradients $g_t^1, g_t^2, \ldots, g_t^M$ at time $t \geq 0$, we define
\begin{align*}
    g_t \eqdef \avemm g_t^m, && \bar{g}_t^m \eqdef \ec{g_t^m} = \begin{cases}
     \nabla f(x_t^m) & \text { for identical data, }  \\
     \nabla f_m (x_t^m) & \text { otherwise, } 
\end{cases} && \bar{g}_t \eqdef \ec{g_t}.
\end{align*}
Let us define two definitions, which are crucial for our analysis
 $$V_t \eqdef \avemm \sqn{x_t^m - \hat{x}_t}_{\hat{D^{t_p}}} \text{ with } t_p \leq t < t_{p+1}; \qquad \hat{x}_t \eqdef \avemm x_t^m.$$
Throughout the proofs, we use the variance decomposition that holds for any random vector $X$ with finite second moment:
\begin{align}
    \label{pr:variance_def}
    \ecn {X}{} = \ecn { X - \ec{X} }{} + \sqn{\ec{X}}.
\end{align}
In particular, its version for vectors with finite number of values gives
\begin{align}
    \avemm \norm{X_m}^2
    = \avemm \norm{X_m - \aveim X_i}^2 + \norm{\avemm X_m}^2.\label{pr:variance_m}
\end{align}
As a consequence of \eqref{pr:variance_def} we have that,
\begin{align}
    \label{pr:variance_sqnorm_upperbound}
    \ecn{X - \ec{X}}{} \leq \ecn{X}{}.
\end{align}
For any convex function $f$ and any vectors $x^1,\dotsc, x^M$ we have Jensen's inequality:
\begin{align}
\label{pr:jensen}
    f\br{\avemm x^m}
    \le \avemm f(x^m).
\end{align}
As a special case with $f(x)=\|x\|^2$, we obtain
\begin{align}
\label{pr:spec-jensen}
    \norm{\avemm x_m}^2
    \le \avemm \|x_m\|^2.
\end{align}
\noindent We denote the Bregman divergence associated with function $f$ and arbitrary $x, y$ as
\begin{align*}
D_f(x, y)
\eqdef f(x) - f(y) - \ev{\nabla f(y), x - y}.
\end{align*}    
If $f$ is $L$-smooth and convex, then for any $x$ and $y$ it holds
\begin{align}
    \label{pr:bregman_div}
    \|\nabla f(x) - \nabla f(y)\|^2
    \le 2LD_f(x, y). 
\end{align}
\noindent If $f$ satisfies Assumption~\ref{asm:convexity-and-smoothness}, then
\begin{equation}
     \label{pr:asm-strong-convexity}
     f(x) + \ev{\nabla f(y), x - y} + \frac{\mu}{2} \sqn{y - x} \leq f(y), \qquad \forall x, y \in \R^d.
\end{equation}
\noindent We also use the following facts:
\begin{align}
    \sqn{x + y} &\leq 2 \sqn{x} + 2 \sqn{y}, \label{pr:sum_sqnorm} \\
    2 \ev{a, b} &\leq \zeta \sqn{a} + \zeta^{-1} \sqn{b}, \text { for all } a, b \in \R^d \text { and } \zeta > 0, \label{pr:youngs-inequality}\\
    \left( 1 - \frac{p}{2}\right)^{-1} &\leq (1 + p), \text{ for all } p \in \left[0, 1\right].
    \label{pr: algebraic}
\end{align}

\section{Proof of \cref{cor:equal-convergence}}
\begin{proof}
    Let us consider two cases:
    \begin{enumerate}
        \item $t_p \leq t < t_{p+1} - 1$ for some $p \in \mathbb{N}$.\\
        In this case, matrices $\hat{D}^{t}$ and $\hat{D}^{t+1}$ are equal by construction. Hence, the fact above is obvious.
        \item $t = t_{p+1} - 1$ for some $p \in \mathbb{N}$.\\
        Here we have a change of the matrix which generates a norm (it becomes new at the iteration $t_{p+1}$). But this fact is obvious due to \cref{lemma: from-spp}.
    \end{enumerate}
    Since we can have no more cases, the above ends the proof.
\end{proof}

\section{Proofs for identical data}\label{ident}

\subsection{Auxiliary lemmas}
\begin{lemma}
    \label{lemma:variance-bound}
    Assume that for any $t: t_p \leq t < t_{p + 1}$ we have $\hat{D}^{t_p}$.
    Under Assumptions \ref{asm:convexity-and-smoothness}, \ref{asm:uniformly-bounded-variance} and \ref{asm: preconditioner property}, we have for 
    Algorithm \ref{alg:local_sgd_with_preconditioner}, which run for identical data with
    $\gamma \leq \frac{\alpha}{2L}$ and with $\abs{t_p - t_{t+1}} \leq H$:
    \begin{align*}
        \ec{V_{t}} \leq \br{H - 1} \frac{\gamma^2 \sigma^2}{\alpha}.
    \end{align*}
\end{lemma}
\begin{proof}
     Let $t \in \N$ be such that $t_p \leq t \leq t_{p+1} - 1$. Recall that for a time $t$ such that $t_p \leq t < t_{p+1}$ we have $x_{t+1}^m = x_t^m - \gamma (\hat{D}^{t_p})^{-1}g_{t}^m$ and $\hat{x}_{t+1} = \hat{x}_t - \gamma (\hat{D}^{t_p})^{-1}g_t$. Hence, for the expectation conditional on $x_t^1, x_t^2, \ldots, x_t^M$ (we use $\E \left[\cdot \right] = \E \left[\cdot| x_t^1, x_t^2, \ldots, x_t^M\right]$ for brevity) we have:
     \begin{align*}
         \ecn{x_{t+1}^m - \hat{x}_{t+1}}{\hat{D}^{t_p}} =&\ \ecn{x_t^m - \gamma (\hat{D}^{t_p})^{-1}g_{t}^m - \hat{x}_t + \gamma (\hat{D}^{t_p})^{-1}g_t}{\hat{D}^{t_p}}
         \\ =&\ \sqn{x_t^m - \hat{x}_t}_{\hat{D}^{t_p}} \\
         &+ \gamma^2 \ecn{(\hat{D}^{t_p})^{-1}\nabla f(x_t^m, z_m) - (\hat{D}^{t_p})^{-1}g_t}{\hat{D}^{t_p}} \\&- 2 \gamma \ec{\ev{x_t^m - \hat{x}_t, (\hat{D}^{t_p})^{-1}\nabla f(x_t^m, z_m) - (\hat{D}^{t_p})^{-1}g_t}_{\hat{D}^{t_p}}} \\
         =&\ \sqn{x_t^m - \hat{x}_t}_{\hat{D}^{t_p}} \\
         &+ \gamma^2 \ecn{(\hat{D}^{t_p})^{-1}\nabla f(x_t^m, z_m) - (\hat{D}^{t_p})^{-1}g_t}{\hat{D}^{t_p}} \\&- 2 \gamma \ev{x_t^m - \hat{x}_t, \ec{\nabla f(x_t^m, z_m) - g_t}} \\
         =&\ \sqn{x_t^m - \hat{x}_t}_{\hat{D}^{t_p}} \\ &+ \gamma^2 \ecn{(\hat{D}^{t_p})^{-1}\nabla f(x_t^m, z_m) - (\hat{D}^{t_p})^{-1}g_t}{\hat{D}^{t_p}}  \\ &- 2 \gamma \ev{x_t^m - \hat{x}_t, \nabla f(x_{t}^m)} \\
         &+ 2 \gamma \ev{x_t^m - \hat{x}_t, \overline{g}_t}.
     \end{align*}
     Averaging both sides over $M$ and noting that $V_t = \frac{1}{M} \sum\limits_{m=1}^M \sqn{x_{t}^m - \hat{x}_t}_{\hat{D}^{t_p}}$, we have
     \begin{align}
         \ec{V_{t+1}} =&\ V_t + \frac{\gamma^2}{M} \sum\limits_{m=1}^M \ecn{(\hat{D}^{t_p})^{-1}\nabla f(x_t^m, z_m) - (\hat{D}^{t_p})^{-1}g_t}{\hat{D}^{t_p}} \nonumber \\ &- \frac{2 \gamma}{M} \sum\limits_{m=1}^M \ev{x_t^m - \hat{x}_t, \nabla f(x_t^m)} + 2 \gamma \underbrace{\ev{ \hat{x}_t - \hat{x}_t, \overline{g}_t}}_{=0} \nonumber \\
         \label{iterate-variance-bound-recursion}
         =&\ V_t + \frac{\gamma^2}{M} \sum\limits_{m=1}^M \ecn{\nabla f(x_t^m, z_m) - g_t}{(\hat{D}^{t_p})^{-1}} \nonumber \\ &- \frac{2 \gamma}{M} \sum\limits_{m=1}^M \ev{x_t^m - \hat{x}_t, \nabla f(x_t^m)}.
     \end{align}
     Now remark that by expanding the square we have,
     \begin{align}
         \label{iterate-gradient-variance-bound-1}
         \mathbb{E}\Big[ ||\nabla f(x_t^m, z_m) - g_t||^2_{(\hat{D}^{t_p})^{-1}}\Big] \nonumber =&\
         \ecn{\nabla f(x_t^m, z_m) - \overline{g}_t + \overline{g}_t - g_t}{(\hat{D}^{t_p})^{-1}} \nonumber \\ =&\
         \ecn{\nabla f(x_t^m, z_m) - \overline{g}_t}{(\hat{D}^{t_p})^{-1}} \nonumber \\&+ \ecn{\overline{g}_t - g_t}{(\hat{D}^{t_p})^{-1}} \nonumber \\ &+ 2 \ec{\ev{\nabla f(x_t^m, z_m) - \overline{g}_t, \overline{g}_t - g_t}_{(\hat{D}^{t_p})^{-1}}}.
     \end{align}
     We decompose the first term in the last equality again by expanding the square and using that $\ec{\nabla f(x_t^m, z_m)} = \nabla f(x_t^m)$,
     \begin{align*}
         \mathbb{E}\Big[ ||\nabla f(x_t^m, z_m) &- \overline{g}_t||^2_{(\hat{D}^{t_p})^{-1}}\Big]  \\ 
         =&\ \ecn{\nabla f(x_t^m, z_m) - \nabla f(x_t^m)}{(\hat{D}^{t_p})^{-1}} \\ &+ \sqn{\nabla f(x_t^m) - \overline{g}_t}_{(\hat{D}^{t_p})^{-1}} \\ &+ 2 \ec{\ev{ \nabla f(x_t^m, z_m) - \nabla f(x_t^m), \nabla f(x_t^m) - \overline{g}_t }_{(\hat{D}^{t_p})^{-1}}} \\
         =&\ \ecn{\nabla f(x_t^m, z_m) - \nabla f(x_t^m)}{(\hat{D}^{t_p})^{-1}} \\ &+ \sqn{\nabla f(x_t^m) - \overline{g}_t}_{(\hat{D}^{t_p})^{-1}} \\ &+ 2 \ev{\ec{ \nabla f(x_t^m, z_m) - \nabla f(x_t^m)}, \nabla f(x_t^m) - \overline{g}_t }_{(\hat{D}^{t_p})^{-1}} \\
         =&\ \ecn{\nabla f(x_t^m, z_m) - \nabla f(x_t^m)}{(\hat{D}^{t_p})^{-1}} \\ &+ \sqn{\nabla f(x_t^m) - \overline{g}_t}_{(\hat{D}^{t_p})^{-1}} \\ &+ 2 \underbrace{\ev{\nabla f(x_t^m) - \nabla f(x_t^m), \nabla f(x_t^m) - \overline{g}_t }_{(\hat{D}^{t_p})^{-1}}}_{=0} \\ =&\ \ecn{\nabla f(x_t^m, z_m) - \nabla f(x_t^m)}{(\hat{D}^{t_p})^{-1}} \\ &+ \sqn{\nabla f(x_t^m) - \overline{g}_t}_{(\hat{D}^{t_p})^{-1}}.
     \end{align*}
     Plugging this into~\eqref{iterate-gradient-variance-bound-1}, we can obtain
     \begin{align*}
           \mathbb{E}\Big[ ||\nabla f(x_t^m, z_m) - g_t ||^2_{(\hat{D}^{t_p})^{-1}}\Big] =&\ \ecn{\nabla f(x_t^m, z_m) - \nabla f(x_t^m)}{(\hat{D}^{t_p})^{-1}} \\ &+ \sqn{\nabla f(x_t^m) - \overline{g}_t}_{(\hat{D}^{t_p})^{-1}} \\ &+ \ecn{\overline{g}_t - g_t}{(\hat{D}^{t_p})^{-1}} \\ &+ 2 \ec{\ev{\nabla f(x_t^m, z_m) - \overline{g}_t, \overline{g}_t - g_t}_{(\hat{D}^{t_p})^{-1}}}.
     \end{align*}
     Averaging over $M$ gives
     \begin{align*}
        \frac{1}{M} \sum\limits_{m=1}^M \mathbb{E}\Big[ || \nabla f(x_t^m, z_m) &- g_t ||^2_{(\hat{D}^{t_p})^{-1}}\Big] \\=&\ \frac{1}{M} \sum\limits_{m=1}^M \ecn{\nabla f(x_t^m, z_m) - \nabla f(x_t^m)}{(\hat{D}^{t_p})^{-1}} \\ &+ \frac{1}{M} \sum\limits_{m=1}^M \sqn{\nabla f(x_t^m) - \overline{g}_t}_{(\hat{D}^{t_p})^{-1}} \\ &+ \ecn{\overline{g}_t - g_t}{(\hat{D}^{t_p})^{-1}} \\
         &- 2 \ecn{\overline{g}_t - g_t}{(\hat{D}^{t_p})^{-1}},
     \end{align*}
     where we used the notation $g_t = \sum\limits_{m=1}^M \nabla f(x_t^m, z_m)$. Hence,
     \begin{align}
         \label{iterate-gradient-variance-bound-2}
         \frac{1}{M} \sum\limits_{m=1}^M \mathbb{E}\Big[ ||\nabla f(x_t^m, z_m) &-g_t||^2_{(\hat{D}^{t_p})^{-1}}\Big] \nonumber \\ \leq&\ \frac{1}{M} \sum\limits_{m=1}^M \ecn{\nabla f(x_t^m, z_m) - \nabla f(x_t^m)}{(\hat{D}^{t_p})^{-1}}  \nonumber \\ &+ \frac{1}{M} \sum\limits_{m=1}^M \sqn{\nabla f(x_t^m) - \overline{g}_t}_{(\hat{D}^{t_p})^{-1}}.
     \end{align}   
     Now we can note that for the first term in~\eqref{iterate-gradient-variance-bound-2} with Assumption \ref{asm:uniformly-bounded-variance}, using that $\sqn{x}_{(\hat{D}^{t_p})^{-1}} \leq \frac{1}{\alpha}\sqn{x}$, we have
     \begin{align}
     \label{iterate-gradient-variance-bound-3}
         \mathbb{E}\Big[ ||&\nabla f(x_t^m, z_m) - \nabla f(x_t^m)||^2_{(\hat{D}^{t_p})^{-1}}\Big] \leq \frac{1}{\alpha} \ecn{\nabla f(x_t^m, z_m) - \nabla f(x_t^m)}{}\leq \frac{\sigma^2}{\alpha}.
     \end{align}
     For the second term in \eqref{iterate-gradient-variance-bound-2}, we get
     \begin{align*}
         ||\nabla f(x_t^m) - \overline{g}_t||^2_{(\hat{D}^{t_p})^{-1}} =&\ \sqn{\nabla f(x_t^m) - \nabla f(\hat{x}_t)}_{(\hat{D}^{t_p})^{-1}} \\ &+ \sqn{\nabla f(\hat{x}_t) - \overline{g}_t }_{(\hat{D}^{t_p})^{-1}} \\
          &+ 2 \ev{\nabla f(x_t^m) - \nabla f(\hat{x}_t), \nabla f(\hat{x}_t) - \overline{g}_t  }_{(\hat{D}^{t_p})^{-1}}.
     \end{align*}
     Averaging over $M$ and using the notation $\bar{g}_t = \frac{1}{M} \sum\limits_{m=1}^M \nabla f(x_t^m)$, we have
     \begin{align*}
         \frac{1}{M} \sum_{m=1}^{M} ||\nabla f(x_t^m) - \overline{g}_t||^2_{(\hat{D}^{t_p})^{-1}} =&\ \frac{1}{M} \sum\limits_{m=1}^M \sqn{\nabla f(x_t^m) - \nabla f(\hat{x}_t)}_{(\hat{D}^{t_p})^{-1} } \\ &+ \sqn{ \nabla f(\hat{x}_t) - \overline{g}_t}_{(\hat{D}^{t_p})^{-1}} \\
         &- 2 \sqn{\nabla f(\hat{x}_t) - \overline{g}_t}_{(\hat{D}^{t_p})^{-1}} \\
         \leq&\ \frac{1}{M} \sum\limits_{m=1}^M \sqn{\nabla f(x_t^m) -  \nabla f(\hat{x}_t)}_{(\hat{D}^{t_p})^{-1}}.
     \end{align*}
     Then, using the property of matrix $\hat{D}^{t_p}$ that $\sqn{x}_{(\hat{D}^{t_p})^{-1}} \leq \frac{1}{\alpha}\sqn{x}$, the assumption about $L$-smoothness (see \cref{asm:convexity-and-smoothness}) and Jensen's inequality, we get
     \begin{align}
         \frac{1}{M} \sum\limits_{m=1}^M ||\nabla f(x_t^m) &- \nabla f(\hat{x}_t)||^2_{(\hat{D}^{t_p})^{-1}} \nonumber \\ \leq&\ \frac{1}{M\alpha} \sum\limits_{m=1}^M \sqn{\nabla f(x_{t}^m) - \nabla f(\hat{x}_t) } \nonumber \\
         \overset{\eqref{pr:bregman_div}}{\le}&\ \frac{1}{M\alpha} \sum\limits_{m=1}^M 2L (f(\hat x_t) - f(x_t^m) - \ev{\hat x_t - x_t^m, \nabla f(x_t^m)}) \nonumber \\
         \label{eq:iterate-gradient-variance-bound-4}
         \le&\ \frac{2L}{M\alpha} \sum\limits_{m=1}^M \ev{x_t^m - \hat x_t, \nabla f(x_t^m)}.
     \end{align}
     Plugging \eqref{eq:iterate-gradient-variance-bound-4} and \eqref{iterate-gradient-variance-bound-3} into \eqref{iterate-gradient-variance-bound-2}, we obtain
     \begin{align}
         \label{iterate-gradient-variance-bound-5}
         \frac{1}{M} \sum\limits_{m=1}^M \mathbb{E}\Big[ ||\nabla f(x_t^m, z_m) &- g_t||^2_{(\hat{D}^{t_p})^{-1}}\Big]  \leq \frac{\sigma^2}{\alpha}  + \frac{2L}{M\alpha} \sum\limits_{m=1}^M \ev{x_t^m - \hat x_t, \nabla f(x_t^m)}.
     \end{align}
     Substituting \eqref{iterate-gradient-variance-bound-5} into \eqref{iterate-variance-bound-recursion}, we get
     \begin{align*}
         \ec{V_{t+1}} 
         \le&\ V_t + \frac{\gamma^2 \sigma^2}{\alpha} + \frac{2L \gamma^2}{M\alpha} \sum\limits_{m=1}^M \ev{x_t^m - \hat x_t, \nabla f(x_t^m)}   \\ &- \frac{2 \gamma}{M} \sum\limits_{m=1}^M \ev{x_t^m - \hat{x}_t, \nabla f(x_t^m)} \nonumber \\
         \label{iterate-variance-recursion-pre-sc}
         =&\ V_t + \frac{\gamma^2 \sigma^2}{\alpha} - \left(1 - \frac{\gamma L}{\alpha}\right)\frac{2 \gamma}{M} \sum\limits_{m=1}^M \ev{x_t^m - \hat{x}_t, \nabla f(x_t^m)} \\
         {\le}&\ V_t + \frac{\gamma^2 \sigma^2}{\alpha} + \left(1 - \frac{\gamma L}{\alpha}\right)\frac{2 \gamma}{M} \sum\limits_{m=1}^M (f(\hat{x}_t) - f(x_t^m) - \frac{\mu}{2}\sqn{x_t^m - \hat{x}_t}) \\
         \leq&\ V_t + \frac{\gamma^2 \sigma^2}{\alpha} + \left(1 - \frac{\gamma L}{\alpha}\right)\frac{2 \gamma}{M} \sum\limits_{m=1}^M (f(\hat{x}_t) - f(x_t^m) - \frac{\mu}{2\Gamma}\sqn{x_t^m - \hat{x}_t}_{\hat{D}^{t_p}}) \\
         \leq&\ \br{1 - \gamma\left(1 - \frac{\gamma L}{\alpha}\right)\frac{\mu}{\Gamma}} V_t + \frac{\gamma^2\sigma^2}{\alpha}, \nonumber
     \end{align*}
      where we used $\mu$-strong convexity, the fact that $\frac{1}{\Gamma}\sqn{x}_{\hat{D}^{t_p}} \leq \sqn{x}$ and the Jensen's inequality. Using that $\gamma \leq \frac{\alpha}{2L}$, we can conclude,
     \begin{align*}
         \ec{V_{t+1}}
         &\le \br{1 - \frac{\gamma\mu}{2\Gamma}} V_t + \frac{\gamma^2\sigma^2}{\alpha} \\
         &\le V_t + \frac{\gamma^2\sigma^2}{\alpha}.
     \end{align*}
     Taking the full expectation and iterating the above inequality,
     \begin{align*}
         \ec{V_t} \leq&\ \ec{V_{t_p}} + \frac{\gamma^2\sigma^2}{\alpha} \br{t - t_p} \\
         \leq&\ \ec{V_{t_p}} + \frac{\gamma^2\sigma^2}{\alpha} \br{t_{p+1} - t_p - 1} \\
         \leq&\ \ec{V_{t_p}} + \frac{\gamma^2\sigma^2}{\alpha} \br{H - 1}.
     \end{align*}
     It remains to notice that by the design of \cref{alg:local_sgd_with_preconditioner} we have $V_{t_p} = 0$.
\end{proof}

\subsection{\bf Other lemmas}
\begin{lemma}
     \label{lemma:iterate-one-recursion}
     Let $(x_t^m)_{t \geq 0}$ be iterates generated by Algorithm~\ref{alg:local_sgd_with_preconditioner} run with identical data. Suppose that $f$ satisfies Assumption \ref{asm:convexity-and-smoothness}, $\hat{D}^{t_p}$ satisfies Assumption \ref{asm: preconditioner property} and that $\gamma \leq \frac{\alpha}{4L}$. Then, for any $t:t_p \leq t < t_{p+1}$,
     \begin{align}
         \begin{split}
             \ecn{\hat{x}_{t+1} - x_\ast}{\hat{D}^{t_p}} \leq&\ \Big(1-\frac{\gamma\mu}{\Gamma}\Big) \E \sqn{\hat{x}_t-x_\ast}_{\hat{D}^{t_p}} + \frac{\gamma^2}{\alpha}  \E \sqn{g_t - \overline{g}_t} \nonumber \\ &- \frac{\gamma}{2} \ec{D_{f} (\hat{x}_t, x_\ast)} +  \frac{2 \gamma L}{\alpha} V_t.
         \end{split}
     \end{align}
\begin{proof}[\bf Proof of Lemma \ref{lemma:iterate-one-recursion}]
From the update rule we get
\begin{align}
\sqn{\hat{x}_{t+1}-x_\ast}_{\hat{D}^{t_p}} =&\ \sqn{\hat{x}_t - \gamma (\hat{D}^{t_p})^{-1} g_t -x_\ast }_{\hat{D}^{t_p}} \nonumber \\ =&\ \sqn{\hat{x}_t - \gamma (\hat{D}^{t_p})^{-1}g_t -x_\ast - \gamma (\hat{D}^{t_p})^{-1}\overline{g}_t + \gamma (\hat{D}^{t_p})^{-1} \overline{g}_t }_{\hat{D}^{t_p}} \nonumber \\
 =&\ \sqn{\hat{x}_t-x_\ast - \gamma (\hat{D}^{t_p})^{-1} \overline{g}_t}_{\hat{D}^{t_p}} + \gamma^2 \sqn{g_t - \overline{g}_t}_{(\hat{D}^{t_p})^{-1}} \nonumber \\ &+ 2\gamma \ev{\hat{x}_t-x_\ast - \gamma(\hat{D}^{t_p})^{-1}\overline{g}_t, \overline{g}_t-g_t}. 
\end{align}
Observe that
\begin{align}
     ||\hat{x}_t-x_\ast - \gamma(\hat{D}^{t_p})^{-1} \overline{g}_t  ||^2_{\hat{D}^{t_p}} \nonumber =&\ \sqn{\hat{x}_t-x_\ast}_{\hat{D}^{t_p}} + \gamma^2 \sqn{\overline{g}_t}_{(\hat{D}^{t_p})^{-1}} \nonumber \\&- 2 \gamma \ev{\hat{x}_t-x_\ast, \overline{g}_t} \nonumber \\
    =&\ \sqn{\hat{x}_t-x_\ast}_{\hat{D}^{t_p}} + \gamma^2  \sqn{\overline{g}_t}_{(\hat{D}^{t_p})^{-1}} \nonumber \\ &-  \frac{2\gamma}{M} \sum_{m=1}^M \ev{\hat{x}_t-x_\ast, \nabla f(x_t^m)} \nonumber \\
    \leq&\ \sqn{\hat{x}_t-x_\ast}_{\hat{D}^{t_p}} + \frac{\gamma^2}{M} \sum_{m=1}^M \sqn{\nabla f(x_t^m)}_{(\hat{D}^{t_p})^{-1}}  \nonumber \\
    & - \frac{2\gamma}{M} \sum_{m=1}^M \ev{\hat{x}_t - x_t^m + x_t^m -x_\ast, \nabla f(x_t^m)} \nonumber \\
    =&\ \sqn{\hat{x}_t-x_\ast}_{\hat{D}^{t_p}} \nonumber \\ &+  \frac{\gamma^2}{M} \sum_{m=1}^M \sqn{\nabla f(x_t^m) - \nabla f(x_\ast)}_{(\hat{D}^{t_p})^{-1}} \nonumber \\
    &-  \frac{2\gamma}{M} \sum_{m=1}^M \ev{ x_t^m -x_\ast, \nabla f(x_t^m)} \nonumber \\ \label{eq:beforeexp} &-  \frac{2\gamma}{M} \sum_{m=1}^M \ev{\hat{x}_t - x_t^m, \nabla f(x_t^m)},
\end{align}
where we used the fact that $\sqn {\sum_{m=1}^M a_m} \leq M \sum_{m=1}^M \sqn{a_m}$. With the property of $\hat{D}^{t_p}$ that $\sqn{x}_{(\hat{D}^{t_p})^{-1}} \leq \frac{1}{\alpha}\sqn{x} $, we get
\begin{align}
    \label{eq:another-norm}
    \sqn{\nabla f(x_t^m) - \nabla f(x_\ast)}_{(\hat{D}^{t_p})^{-1}} \leq \frac{1}{\alpha}\sqn{\nabla f(x_t^m) - \nabla f(x_\ast)}.
\end{align}
By $L$-smoothness (\cref{asm:convexity-and-smoothness}, see also \eqref{pr:bregman_div}), 
\begin{align}
\sqn{\nabla f(x_t^m)- \nabla f(x_\ast)} \leq 2L(f(x_t^m) - f_\ast)\,, \label{eq:followfromL}
\end{align}
and by $\mu$-strong convexity
\begin{align}
    \label{eq:muconvex}
  -  \ev{x_t^m-x_\ast, \nabla f(x_t^m)} \leq&\ - (f(x_t^m)- f_\ast) - \frac{\mu}{2}\sqn{x_t^m-x_\ast}\,.
\end{align}
To estimate the last term in~\eqref{eq:beforeexp} we use $ 2 \ev{a,b} \leq \gamma \sqn{a} + \gamma^{-1} \sqn{b}$, for $\gamma = 2L > 0$. This gives
\begin{align}
\label{eq:scalarbound}
-2 \ev{\hat{x}_t - x_t^m, \nabla f(x_t^m)} \leq&\ 2L \sqn{\hat{x}_t - x_t^m} + \frac{1}{2L} \sqn{\nabla f(x_t^m) } \nonumber \\
=&\ 2L \sqn{\hat{x}_t - x_t^m} + \frac{1}{2L} \sqn{\nabla f(x_t^m) - \nabla f(x_\ast) } \nonumber \\
\leq&\ 2L \sqn{\hat{x}_t - x_t^m} + (f(x_t^m) - f_\ast)\,,
\end{align} 
where we used~\eqref{eq:followfromL} in the last inequality.
By applying \eqref{eq:another-norm}, \eqref{eq:followfromL}, \eqref{eq:muconvex} and \eqref{eq:scalarbound} to~\eqref{eq:beforeexp}, we get
\begin{align}
||\hat{x}_t-x_\ast - &\gamma (\hat{D}^{t_p})^{-1} \overline{g}_t||^2_{\hat{D}^{t_p}} \nonumber \\ \leq&\  \sqn{\hat{x_t}-x_\ast}_{\hat{D}^{t_p}}  + \frac{2 \gamma L}{M} \sum_{m=1}^M  \sqn{\hat{x}_t - x_t^m} \nonumber \\ & \label{eq:smallL} +  \frac{2 \gamma}{M} \sum_{m=1}^M \left(  \left(\frac{\gamma L}{\alpha} -\frac{1}{2}\right)  (f(x_t^m)- f_\ast) - \frac{\mu}{2}\sqn{x_t^m-x_\ast} \right) \,.  
\end{align}
For $\gamma \leq \frac{\alpha}{4L}$ it holds $\bigl(\frac{\gamma L}{\alpha} -\frac{1}{2}\bigr) \leq -\frac{1}{4}$. 
By applying the Jensen's inequality to the convex function $\left[a \left(f(x) - f_\ast\right)+ b \sqn{x- x_\ast }\right]$ with $a = \frac{1}{2} - \frac{\gamma L}{\alpha} \geq 0$, $b = \frac{\mu}{2} \geq 0$:
\begin{align}
\label{eq:mod-new-func}
- \frac{1}{M}\sum_{m=1}^M \left( a (f(x_t^m)- f_\ast) +  b \sqn{x_t^m-x_\ast} \right) \leq&\  -  \left(a (f(\hat{x}_t)- f_\ast)\right) -  b\sqn{\hat{x}_t-x_\ast},
\end{align}
hence we can continue with~\eqref{eq:smallL} by substituting \eqref{eq:mod-new-func} and by using that $\frac{1}{\alpha}\sqn{x}_{\hat{D}^{t_p}} \geq \sqn{x} \geq \frac{1}{\Gamma}\sqn{x}_{\hat{D}^{t_p}}$:
\begin{align}
\sqn{\hat{x}_t-x_\ast - \gamma (\hat{D}^{t_p})^{-1}  \overline{g}_t}_{\hat{D}^{t_p}}  \leq&\  \Big(1-\frac{\gamma\mu}{\Gamma}\Big) \sqn{\hat{x}_t-x_\ast}_{\hat{D}^{t_p}}  -\frac{\gamma}{2}(f(\hat{x}_t)- f_\ast) \nonumber \\ & \label{eq:barstep} + \frac{2 \gamma L}{M\alpha} \sum_{m=1}^M   \sqn{\hat{x}_t - x_t^m}_{\hat{D}^{t_p}}  \,. 
\end{align}
Plugging \eqref{eq:barstep} and taking the full expectation we get
\begin{align}
\E \sqn{\hat{x}_{t+1}-x_\ast}_{\hat{D}^{t_p}} \leq&\ \Big(1-\frac{\gamma\mu}{\Gamma}\Big) \E \sqn{\hat{x}_t-x_\ast}_{\hat{D}^{t_p}} + \gamma^2 \sqn{g_t - \overline{g}_t}_{(\hat{D}^{t_p})^{-1}} \nonumber \\ &-\frac{\gamma}{2}(f(\hat{x}_t)- f_\ast)  + \frac{2 \gamma L}{M\alpha} \sum_{m=1}^M   \sqn{\hat{x}_t - x_t^m}_{\hat{D}^{t_p}}\ \nonumber \\ \leq&\ \Big(1-\frac{\gamma\mu}{\Gamma}\Big) \E \sqn{\hat{x}_t-x_\ast}_{\hat{D}^{t_p}} + \frac{\gamma^2}{\alpha}  \E \sqn{g_t - \overline{g}_t} \nonumber \\ &- \frac{\gamma}{2} \ec{D_{f} (\hat{x}_t, x_\ast)} +  \frac{2 \gamma L}{M\alpha} \sum_{m=1}^M \E \sqn{\hat{x}_t - x_t^m}_{\hat{D}^{t_p}}.
\end{align}
Using the notation of $V_t$, we claim the final result.
\end{proof}
\end{lemma}

\begin{lemma}
     \label{lemma:minibatch-variance-reduction}
     Suppose that Assumption \ref{asm:uniformly-bounded-variance} holds. Then, if Algorithm \ref{alg:local_sgd_with_preconditioner} runs with identical data, we have
     \[ \ecn{g_t - \bar{g}_t}{} \leq \frac{\sigma^2}{M}. \]
\end{lemma}
\begin{proof}
     Because the stochastic gradients $\nabla f(x_t^m, z_m)$ are independent, according to \ref{asm:uniformly-bounded-variance} we have
     \begin{align*}
         \ecn{g_t - \bar{g}_t}{} =&\ \ecn{\frac{1}{M}\sum\limits_{m=1}^M \nabla f(x_t^m, z_m) - \nabla f(x_t^m)}{}
         \\
         =&\ \frac{1}{M^2} \ecn{ \sum_{m=1}^{M} \nabla f(x_t^m, z_m) - \nabla f(x_t^m) }{} \\ =&\ \frac{1}{M^2} \sum_{m=1}^{M} \ecn{\nabla f(x_t^m, z_m) - \nabla f(x_t^m)}{} \leq \frac{\sigma^2}{M}.
     \end{align*}
\end{proof}

\subsection{\bf Proof of Theorem \ref{thm:ident_convergence_theorem}}
\begin{proof}
     Combining Lemma~\ref{lemma:iterate-one-recursion} and Lemma~\ref{lemma:minibatch-variance-reduction}, we have
     \begin{align}
         \label{sc-thm-proof-1}
         \ecn{\hat{x}_{t+1} - x_\ast}{\hat{D}^{t_p}} \leq&\ \Big(1-\frac{\gamma\mu}{\Gamma}\Big) \E \sqn{\hat{x}_t-x_\ast}_{\hat{D}^{t_p}} + \frac{\gamma^2\sigma^2}{\alpha M} \nonumber \\ &- \frac{\gamma}{2} \ec{D_{f} (\hat{x}_t, x_\ast)} +  \frac{2 \gamma L}{\alpha} V_t.
     \end{align}
     Using Lemma \ref{lemma:variance-bound} we can upper bound the $\ec{V_t}$ term in $\eqref{sc-thm-proof-1}$:
     \begin{align*}
         \ecn{\hat{x}_{t+1} - x_\ast}{\hat{D}^{t_p}} \leq&\ \Big(1-\frac{\gamma\mu}{\Gamma}\Big) \E \sqn{\hat{x}_t-x_\ast}_{\hat{D}^{t_p}} + \frac{\gamma^2\sigma^2}{\alpha M} \nonumber \\ &- \frac{\gamma}{2} \ec{D_{f} (\hat{x}_t, x_\ast)} +  \frac{2 \gamma^3 L}{\alpha^2}(H-1)\sigma^2.
     \end{align*}
     Applying \cref{cor:equal-convergence} and using that $1 + \frac{\gamma\mu}{2\Gamma} \leq 2$, we get
     \begin{align*}
         \ecn{\hat{x}_{t+1} - x_\ast}{\hat{D}^{t+1}} \leq&\ \Big(1-\frac{\gamma\mu}{2\Gamma}\Big) \E \sqn{\hat{x}_t-x_\ast}_{\hat{D}^{t}} + \frac{2\gamma^2\sigma^2}{\alpha M} \nonumber \\ &- \gamma \ec{D_{f} (\hat{x}_t, x_\ast)} +  \frac{4 \gamma^3 L}{\alpha^2}(H-1)\sigma^2.
     \end{align*}
     Due to $\ec{D_{f} (\hat{x}_t, x_\ast)} \geq 0$, we have
     \begin{equation*}
         \ecn{\hat{x}_{t+1} - x_\ast}{\hat{D}^{t+1}} \leq \Big(1-\frac{\gamma\mu}{2\Gamma}\Big) \ecn{\hat{x}_t-x_\ast}{\hat{D}^{t}} + \frac{2\gamma^2 \sigma^2}{\alpha M} + \frac{4\gamma^3 L}{\alpha^2} \br{H - 1} \sigma^2.
     \end{equation*}
     Running the recursion, we can obtain
     \begin{align*}
         \ecn{\hat{x}_T-x_\ast}{\hat{D}^{T}} \leq&\ \br{1 - \frac{\gamma\mu}{2\Gamma}}^{T} \ecn{x_0-x_\ast}{\hat{D}^{0}} \\ &+ \br{ \sum_{t=0}^{T-1} \br{1 - \frac{\gamma\mu}{2\Gamma}}^t } \br{ \frac{2\gamma^2 \sigma^2}{\alpha M} + \frac{4\gamma^3L}{\alpha^2} \br{H - 1} \sigma^2 }.
     \end{align*}
     Using that $\sum\limits_{t=0}^{T-1} \br{1 - \frac{\gamma\mu}{2\Gamma}}^{t} \leq \sum\limits_{t=0}^{\infty} \br{1 - \frac{\gamma\mu}{2\Gamma}}^{t} = \frac{2\Gamma}{\gamma \mu}$,
     \begin{align*}
         \ecn{\hat{x}_T-x_\ast}{\hat{D}^{T}} \leq&\ \br{1 - \frac{\gamma\mu}{2\Gamma}}^{T} \ecn{x_0-x_\ast}{\hat{D}^{0}} + \frac{4\Gamma\gamma \sigma^2}{\mu M \alpha} + \frac{8 \Gamma\gamma^2 L \br{H - 1} \sigma^2}{\mu\alpha^2}.
     \end{align*}
     Using that $\frac{1}{\Gamma}\sqn{x}_{\hat{D}^T} \leq \sqn{x} \leq \frac{1}{\alpha}\sqn{x}_{\hat{D}^T}$, we get
     \begin{align*}
         \ecn{\hat{x}_T-x_\ast}{} \leq&\ \br{1 - \frac{\gamma\mu}{2\Gamma}}^{T} \frac{\Gamma}{\alpha}\ecn{x_0-x_\ast}{} + \frac{4\Gamma\gamma \sigma^2}{\mu M \alpha^2} + \frac{8 \Gamma\gamma^2 L \br{H - 1} \sigma^2}{\mu\alpha^3},
     \end{align*}
     which finishes the proof of the theorem.
 \end{proof}

 \section{\bf Proofs for heterogeneous data} \label{hetero}
\subsection{\bf Auxiliary lemmas}

\begin{lemma}
    \label{lemma:average-gradient-bound}
    Suppose that Assumptions~\ref{asm:convexity-and-smoothness}, \ref{asm:finite-sum-stochastic-gradients} and \ref{asm: preconditioner property} hold with $\mu \geq 0$. Then, if \cref{alg:local_sgd_with_preconditioner} runs for heterogeneous data with $M \geq 2$, we have for any $t:t_p \leq t < t_{p+1}$
    \begin{equation}
        \label{eq:lma-average-gradient-bound}
        \ecn{\frac{1}{M}\sum\limits_{m=1}^M \nabla f_m(x_t^m, z_m)}{(\hat{D}^{t_p})^{-1}} \leq  \frac{2L^2}{\alpha^2} V_t + \frac{8L}{\alpha} D_{f} (\hat{x}_t, x_\ast) + \frac{4 \sigmaf^2}{M\alpha}.
    \end{equation}
\end{lemma}
\begin{proof}
    Starting with the left-hand side, we get
    \begin{align}
        \label{eq:lma-agb-proof-1}
        \mathbb{E}\Bigg[ \bigg{|} \bigg{|}\frac{1}{M}\sum\limits_{m=1}^M &\nabla f_m(x_t^m, z_m)\bigg{|} \bigg{|}^2_{(\hat{D}^{t_p})^{-1}}\Bigg] \nonumber \\ \leq&\ 2 \ecn{\frac{1}{M}\sum\limits_{m=1}^M \nabla f_m(x_t^m, z_m) - \frac{1}{M} \sum_{m=1}^{M}\nabla f_m (\hat{x}_t, z_m) }{(\hat{D}^{t_p})^{-1} } \nonumber \\ &+ 2 \ecn{ \frac{1}{M} \sum_{m=1}^{n} \nabla f_m (\hat{x}_t, z_m) }{(\hat{D}^{t_p})^{-1}}.
    \end{align}
    To bound the first term in \eqref{eq:lma-agb-proof-1}, we need to use the $L$-smoothness of $f_m (\cdot, z_m)$ and the fact that $\sqn{x}_{(\hat{D}^{t_p})^{-1}} \leq \frac{1}{\alpha}\sqn{x}$,
    \begin{align}
       2 \mathbb{E}\Bigg[ \bigg{|} \bigg{|} \frac{1}{M} \sum_{m=1}^{M} &\nabla f_m (x_t^m, z_m) -  \nabla f_m (\hat{x}_t, z_m)\bigg{|} \bigg{|}^2_{(\hat{D}^{t_p})^{-1}}\Bigg] \nonumber \\ \leq&\ \frac{2}{M} \sum_{m=1}^{M} \ecn{(\nabla f_m (x_t^m, z_m) - \nabla f_m (\hat{x}_t, z_m))}{(\hat{D}^{t_p})^{-1}} \nonumber \\
       \leq&\ \frac{2}{M\alpha} \sum_{m=1}^{M} \ecn{(\nabla f_m (x_t^m, z_m) - \nabla f_m (\hat{x}_t, z_m))}{} \nonumber \\
        \label{eq:lma-agb-proof-2}
        \leq&\  \frac{2 L^2}{M\alpha^2} \sum_{m=1}^{M} \sqn{x_t^m - \hat{x}_t}_{\hat{D}^{t_p}},
    \end{align}
    where we used Jensen's inequality and the convexity of the function $\|x\|^2$. For the second term in \eqref{eq:lma-agb-proof-1}, we have
    \begin{align}
        \label{eq:lma-agb-proof-2-2}
            \mathbb{E}\Bigg[ \bigg{|} \bigg{|} \frac{1}{M} \sum_{m=1}^{M} &\nabla f_m (\hat{x}_t, z_m) \bigg{|} \bigg{|}^2_{(\hat{D}^{t_p})^{-1}}\Bigg] \nonumber \\ \overset{\eqref{pr:variance_def}}{=}&\ \ecn{ \frac{1}{M} \sum_{m=1}^{M} \nabla f_m (\hat{x}_t, z_m) - \nabla f_m (\hat{x}_t)}{(\hat{D}^{t_p})^{-1}} \nonumber \\
            &+ \sqn{ \frac{1}{M} \sum_{m=1}^{M} \nabla f_m (\hat{x}_t)}_{(\hat{D}^{t_p})^{-1}}.
    \end{align}
For the first term in \eqref{eq:lma-agb-proof-2-2} by the independence of $z_i$ and by the fact that $\sqn{x}_{(\hat{D}^{t_p})^{-1}} \leq \frac{1}{\alpha}\sqn{x}$, we get
    \begin{align*}
        \mathbb{E}\Bigg[ \bigg{|} \bigg{|} \frac{1}{M} \sum_{m=1}^{M} &\nabla f_m (\hat{x}_t, z_m) - \nabla f_m (\hat{x}_t)\bigg{|} \bigg{|}^2_{(\hat{D}^{t_p})^{-1}}\Bigg]  \\ =&\ \frac{1}{M^2} \sum_{m=1}^{M} \ecn{\nabla f_m (\hat{x}_t, z_m) - \nabla f_m (\hat{x}_t)}{(\hat{D}^{t_p})^{-1}} \\
        \overset{\eqref{pr:variance_sqnorm_upperbound}}{\leq}&\ \frac{1}{M^2} \sum_{m=1}^{M} \ecn{\nabla f_m (\hat{x}_t, z_m)}{(\hat{D}^{t_p})^{-1}} \\
        {\leq}&\ \frac{2}{M^2} \sum_{m=1}^{M} \ecn{\nabla f_m (\hat{x}_t, z_m) - \nabla f_m (x_\ast, z_m)}{(\hat{D}^{t_p})^{-1}} \\ &+ \frac{2}{M^2} \sum_{m=1}^{M} \ecn{\nabla f_m (x_\ast, z_m)}{(\hat{D}^{t_p})^{-1}} \\
        {\leq}&\ \frac{2}{M^2\alpha} \sum_{m=1}^{M} \ecn{\nabla f_m (\hat{x}_t, z_m) - \nabla f_m (x_\ast, z_m)}{} \\ &+ \frac{2}{M^2\alpha} \sum_{m=1}^{M} \ecn{\nabla f_m (x_\ast, z_m)}{} \\
        \overset{\eqref{pr:bregman_div}}{\leq}&\ \frac{4 L}{M^2 \alpha} \sum_{m=1}^{M} D_{f_m} (\hat{x}_t, x_\ast) + \frac{2 \sigmaf^2}{M\alpha^2}\\
        =&\ \frac{4L}{M\alpha} D_{f} (\hat{x}_t, x_\ast) + \frac{2 \sigmaf^2}{M\alpha}.
    \end{align*}
    Substituting this in \eqref{eq:lma-agb-proof-2-2}, one can obtain
    \begin{align}
        \label{eq:important-but-not}
          \mathbb{E}\Bigg[ \bigg{|} \bigg{|}  \frac{1}{M} \sum_{m=1}^{M} &\nabla f_m (\hat{x}_t, z_m) \bigg{|} \bigg{|}^2_{(\hat{D}^{t_p})^{-1}}\Bigg] \nonumber \\ \leq&\ \frac{4L}{M\alpha} D_{f} (\hat{x}_t, x_\ast) + \frac{2 \sigmaf^2}{M\alpha} + \ecn{ \frac{1}{M} \sum_{m=1}^{M}  \nabla f_m (\hat{x}_t)}{(\hat{D}^{t_p})^{-1}} \nonumber \\
         =&\ \frac{4L}{M\alpha} D_{f} (\hat{x}_t, x_\ast) + \frac{2 \sigmaf^2}{M\alpha} + \sqn{ \nabla f(\hat{x}_t)}_{(\hat{D}^{t_p})^{-1}}.
    \end{align}
    Now notice that
    \begin{align*}
    \sqn{\nabla f(\hat{x}_t)}_{(\hat{D}^{t_p})^{-1}} &\leq \frac{1}{\alpha}\sqn{\nabla f(\hat{x}_t)} \\ &= \frac{1}{\alpha}\sqn{\nabla f(\hat{x}_t) - \nabla f(x_\ast)} \leq \frac{2L}{\alpha} D_{f} (\hat{x}_t, x_\ast).
    \end{align*}
    Using this in \eqref{eq:important-but-not}, we get
    \begin{align*}
        \mathbb{E}\Bigg[ \bigg{|} \bigg{|}  \frac{1}{M} \sum_{m=1}^{M} &\nabla f_m (\hat{x}_t, z_m) \bigg{|} \bigg{|}^2_{(\hat{D}^{t_p})^{-1}}\Bigg] \leq \frac{2L}{\alpha} \br{1 + \frac{2}{M}} D_{f} (\hat{x}_t, x_\ast) + \frac{2 \sigmaf^2}{M\alpha}.
    \end{align*}
    Since $M \geq 2$, we have $1 + \frac{2}{M} \leq 2$, and hence
    \begin{align}
        \label{eq:lma-agb-proof-3}
        \mathbb{E}\Bigg[ \bigg{|} \bigg{|}  \frac{1}{M} \sum_{m=1}^{M} &\nabla f_m (\hat{x}_t, z_m) \bigg{|} \bigg{|}^2_{(\hat{D}^{t_p})^{-1}}\Bigg] \leq \frac{4L}{\alpha} D_{f} (\hat{x}_t, x_\ast) + \frac{2 \sigmaf^2}{M\alpha}.
    \end{align}
    Combining \eqref{eq:lma-agb-proof-2} and \eqref{eq:lma-agb-proof-3} in \eqref{eq:lma-agb-proof-1}, we have
    \begin{align*}
        \ecn{\frac{1}{M}\sum\limits_{m=1}^M \nabla f_m(x_t^m, z_m)}{(\hat{D}^{t_p})^{-1}} \leq \frac{2L^2}{\alpha^2} V_t + \frac{8L}{\alpha} D_{f} (\hat{x}_t, x_\ast) + \frac{4 \sigmaf^2}{M\alpha},
    \end{align*}
    which finishes the proof.
\end{proof}
\begin{lemma}
    \label{lemma:inner-product-bound}
    Suppose that Assumptions~\ref{asm:convexity-and-smoothness} and \ref{asm: preconditioner property} hold with $\mu \geq 0$. Then, if \cref{alg:local_sgd_with_preconditioner} runs with heterogeneous data, we have for any $t:t_p \leq t < t_{p+1}$
    \begin{align*}
         -\frac{2 }{M} \sum_{m=1}^{M} \ev{\hat{x}_t - x_\ast, \nabla f_m (x_t^m)} \leq - 2  D_{f} (\hat{x}_t, x_\ast) -  \frac{\mu}{\Gamma} \sqn{\hat{x}_t - x_\ast}_{\hat{D}^{t_p}} +  \frac{L}{\alpha} V_t.
    \end{align*}
\end{lemma}
\begin{proof}
    Starting with the left-hand side,
    \begin{align}
        \label{eq:lma-inner-prod-proof-1}
        - 2  \ev{\hat{x}_t - x_\ast, \nabla f_m (x_t^m)} =&\ - 2  \ev{x_t^m - x_\ast, \nabla f_m (x_t^m)} \nonumber \\  &- 2  \ev{\hat{x}_t - x_t^m, \nabla f_m (x_t^m)}.
    \end{align}
    The first term in \eqref{eq:lma-inner-prod-proof-1} is bounded by strong convexity:
    \begin{align}
        \label{eq:lma-inner-prod-proof-2}
        - \ev{x_t^m - x_\ast, \nabla f_m (x_t^m)} \leq&\ f_m (x_\ast) - f_m (x_t^m) - \frac{\mu}{2} \sqn{x_t^m - x_\ast}.
    \end{align}
    For the second term, we use $L$-smoothness,
    \begin{align}
        \label{eq:lma-inner-prod-proof-3}
        - \ev{\hat{x}_t  - x_t^m, \nabla f_m (x_t^m)} \leq f_m (x_t^m) - f_m (\hat{x}_t) + \frac{L}{2} \sqn{x_t^m - \hat{x}_t}.
    \end{align}
    Combining \eqref{eq:lma-inner-prod-proof-3} and \eqref{eq:lma-inner-prod-proof-2} in \eqref{eq:lma-inner-prod-proof-1}, we get
    \begin{align*}
        - 2  \langle\hat{x}_t - x_\ast, \nabla f_m (x_t^m) \rangle  \leq&\ 2  \br{f_m (x_\ast) - f_m (x_t^m) - \frac{\mu}{2} \sqn{x_t^m - x_\ast}} \\
        &+ 2  \br{f_m (x_t^m) - f_m (\hat{x}_t) + \frac{L}{2} \sqn{x_t^m - \hat{x}_t}} \\
        =&\ 2  \br{f_m (x_\ast) - f_m (\hat{x}_t) - \frac{\mu}{2} \sqn{x_t^m - x_\ast} + \frac{L}{2} \sqn{x_t^m - \hat{x}_t}}.
    \end{align*}
    Averaging over $M$,
    \begin{align*}
        -\frac{2 }{M} \sum_{m=1}^{M} \ev{\hat{x}_t - x_\ast, \nabla f_m (x_t^m)} \leq&\ - 2  \br{f(\hat{x}_t) - f(x_\ast)} \\ &- \frac{ \mu}{M} \sum_{m=1}^{M} \sqn{x_t^m - x_\ast} \\ &+ \frac{ L}{M} \sum_{m=1}^{M} \sqn{x_t^m - \hat{x}_t}.
    \end{align*}
    Noting that the first term is the Bregman divergence $D_{f} (\hat{x}_t, x_\ast)$, and using Jensen's inequality: $- \frac{1}{M} \sum_{m=1}^{M} \sqn{x_t^m - x_\ast} \leq - \sqn{\hat{x}_t - x_\ast}$, we have
    \begin{align*}
        -\frac{2 }{M} \sum_{m=1}^{M} \ev{\hat{x}_t - x_\ast, \nabla f_m (x_t^m)} \leq&\ - 2  D_{f} (\hat{x}_t, x_\ast) -  \mu \sqn{\hat{x}_t - x_\ast} +  \frac{ L}{M} \sum_{m=1}^{M} \sqn{x_t^m - \hat{x}_t}.
    \end{align*}
    With the fact that $\frac{1}{\Gamma}\sqn{x}_{\hat{D}^{t_p}} \leq \sqn{x} \leq \frac{1}{\alpha}\sqn{x}_{\hat{D}^{t_p}}$ and with the notation of $V_t$, one can obtain
     \begin{align*}
        -\frac{2 }{M} \sum_{m=1}^{M} \ev{\hat{x}_t - x_\ast, \nabla f_m (x_t^m)} \leq&\ - 2  D_{f} (\hat{x}_t, x_\ast) -  \frac{\mu}{\Gamma} \sqn{\hat{x}_t - x_\ast}_{\hat{D}^{t_p}} +  \frac{L}{\alpha} V_t.
    \end{align*}
    This concludes the proof of the lemma.
\end{proof}

\begin{lemma}
    \label{lemma:iterate-deviation-epoch}
    Suppose that Assumptions~\ref{asm:convexity-and-smoothness}, \ref{asm:finite-sum-stochastic-gradients} and \ref{asm: preconditioner property} hold. Then, if Algorithm~\ref{alg:local_sgd_with_preconditioner} runs for heterogeneous data with $\sup_{p} \abs{t_p - t_{p+1}} \leq H$, we get for any $t: t_p \leq t \leq t_{p+1} - 1$ with $\gamma \leq \frac{\alpha}{3 L \br{H - 1}}$ and $w_j \eqdef \br{1 - \frac{\gamma\mu}{2\Gamma}}^{-(j+1)}$
    \begin{align}
        \sum\limits_{j=t_p}^t w_j \ec{V_j} \leq&\  \frac{36\gamma^2(H-1)^2 L}{\alpha}\sum\limits_{j=t_p}^t w_j  \ec{D_{f} (\hat{x}_k, x_\ast)} \nonumber \\
       &+  \frac{18\gamma^2(H-1)^2  \sigmaf^2}{\alpha}\sum\limits_{j=t_p}^t w_j. \nonumber
    \end{align}
\end{lemma}
\begin{proof}
    Let $G_k \eqdef \sum_{m=1}^M(\hat{D}^{t_p})^{-1}\nabla f_m(x_k^m, z_m)$. From the notation of $V_t$, we get
    \begin{align*}
        &\ec{V_t} = \frac{1}{M} \sum_{m=1}^{M} \ecn{x_t^m - \hat{x}_t}{\hat{D}^{t_p}} \\
        =&\ \frac{1}{M} \sum_{m=1}^{M} \ecn{ \br{x_{t_p}^{m} - \gamma \sum_{k=t_p}^{t-1} (\hat{D}^{t_p})^{-1}\nabla f_m(x_k^m, z_m) } - \br{x_{t_p} - \gamma \sum_{k=t_p}^{t-1} G_k} }{\hat{D}^{t_p}}.
    \end{align*}
    Using that $x_{t_p} = x_{t_p}^{m}$ for all $m \in \left[M\right]$,
    \begin{align}
    \label{eq:vt-bound}
        \ec{V_t} =&\ \frac{\gamma^2}{M} \sum_{m=1}^{M} \ecn{ \sum_{k=t_p}^{t-1} \nabla f_m(x_k^m, z_m) - g_k }{(\hat{D}^{t_p})^{-1}} \nonumber \\
        \overset{\eqref{pr:spec-jensen}}{\leq}&\ \frac{\gamma^2 \br{t - t_p}}{M} \sum_{m=1}^{M} \sum_{k=t_p}^{t-1} \ecn{\nabla f_m(x_k^m, z_m) - g_k}{(\hat{D}^{t_p})^{-1}} \nonumber \\ \leq&\ \frac{\gamma^2 \br{t - t_p}}{M} \sum_{m=1}^{M} \sum_{k=t_p}^{t-1} \ecn{\nabla f_m(x_k^m, z_m)}{(\hat{D}^{t_p})^{-1}} \nonumber \\
        \leq&\ \frac{\gamma^2 \br{H - 1}}{M} \sum_{m=1}^{M}  \sum_{k=t_p}^{t-1} \ecn{\nabla f_m(x_k^m, z_m)}{(\hat{D}^{t_p})^{-1}},
    \end{align}
    where in the third line we used the definition of 
 $g_k$ in the following way $$\avemm \ecn{\nabla f_m(x_k^m, z_m) - g_k}{(\hat{D}^{t_p})^{-1}} \leq \avemm \ecn{\nabla f_m(x_k^m, z_m)}{(\hat{D}^{t_p})^{-1}},$$ and in the fourth line we used that $t - t_p \leq t_{p+1} - t_p - 1 \leq H - 1$.  
    Decomposing the gradient norm, one can obtain
    \begin{align}
        \label{eq:lma-ide-proof-2}
        \ecn{\nabla f_m(x_k^m, z_m)}{(\hat{D}^{t_p})^{-1}} \nonumber \leq&\ 3 \ecn{\nabla f_m(x_k^m, z_m) - \nabla f_m (\hat{x}_k, z_m)}{(\hat{D}^{t_p})^{-1}} \nonumber \\&+ 3 \ecn{\nabla f_m (\hat{x}_k, z_m) - \nabla f_m (x_\ast, z_m)}{(\hat{D}^{t_p})^{-1}} \nonumber \\
        &+ 3 \ecn{\nabla f_m (x_\ast, z_m)}{(\hat{D}^{t_p})^{-1}}.
    \end{align}
    For the first term in \eqref{eq:lma-ide-proof-2}:
    \begin{align}
        \label{eq:lma-ide-proof-3}
        \mathbb{E}\Big[ \big{|} \big{|}(\nabla f_m(x_k^m, z_m) - \nabla f_m (\hat{x}_t&, z_m)\big{|} \big{|}^2_{(\hat{D}^{t_p})^{-1}}\Big] \nonumber \\ \leq&\ \frac{1}{\alpha}\ecn{\nabla f_m (x_k^m, z_m) - \nabla f_m (\hat{x}_t, z_m)} \nonumber \\  
        \leq&\ \frac{L^2}{\alpha} \ecn{x_k^m - \hat{x}_k}{}.
    \end{align}
    The second term can be bounded by smoothness and the property of $(\hat{D}^{t_p})^{-1}$ that $\sqn{x}_{(\hat{D}^{t_p})^{-1}} \leq \frac{1}{\alpha}\sqn{x}$:
    \begin{align}
        \label{eq:lma-ide-proof-4}
        \ecn{\nabla f_m (\hat{x}_k, z_m) - \nabla f_m (x_\ast, z_m)}{(\hat{D}^{t_p})^{-1}} \leq&\ \frac{1}{\alpha}\ecn{\nabla f_m (\hat{x}_k, z_m) - \nabla f_m (x_\ast, z_m)}{} \nonumber \\  
        \overset{\eqref{pr:bregman_div}}{\leq}&\ \frac{2L}{\alpha} \ec{D_{f_m} (\hat{x}_k, x_\ast)}.
    \end{align}
    Using \eqref{eq:lma-ide-proof-4} and \eqref{eq:lma-ide-proof-3} in \eqref{eq:lma-ide-proof-2}, averaging by $M$, using that $\alpha \sqn{x} \leq \sqn{x}_{\hat{D}^{t_p}}$, with the notation of $\sigmaf^2$ we have
    \begin{align}
        \label{eq:lma-ide-proof-5}
        \avemm \ecn{\nabla f_m(x_k^m, z_m)}{(\hat{D}^{t_p})^{-1}} \leq&\ \frac{3L^2}{M\alpha} \sum_{m=1}^{M} \ecn{x_k^m - \hat{x}_k} \nonumber\\ &+ \frac{6L}{\alpha} \ec{D_{f} (\hat{x}_k, x_\ast)} + \frac{3\sigmaf^2}{\alpha}  \nonumber \\
        \leq&\ \frac{3 L^2}{\alpha^2} \ec{V_k} \nonumber \\ &+ \frac{6 L}{\alpha} \ec{D_{f} (\hat{x}_k, x_\ast)} + \frac{3 \sigmaf^2}{\alpha}.
    \end{align}
   Plugging \eqref{eq:lma-ide-proof-5} into \eqref{eq:vt-bound} and summarizing inequalities with weights $w_j$, we get
   \begin{align}
       \label{eq:lma-ide-proof-6}
       \sum\limits_{j=t_p}^t w_j \ec{V_j} \leq&\ \sum\limits_{j=t_p}^t w_j \gamma^2(H-1) \sum\limits_{k=t_p}^{j - 1}\br{\frac{3 L^2}{\alpha^2} \ec{V_k} + \frac{6 L}{\alpha} \ec{D_{f} (\hat{x}_k, x_\ast)} + \frac{3 \sigmaf^2}{\alpha}} \nonumber \\
       =&\  \gamma^2(H-1) \sum\limits_{j=t_p}^t \sum\limits_{k=t_p}^{j - 1}w_j\br{\frac{3 L^2}{\alpha^2} \ec{V_k} + \frac{6 L}{\alpha} \ec{D_{f} (\hat{x}_k, x_\ast)} + \frac{3 \sigmaf^2}{\alpha}} \nonumber \\
       =&\ \gamma^2(H-1) \sum\limits_{j=t_p}^t \sum\limits_{k=t_p}^{j - 1}w_j \frac{3 L^2}{\alpha^2} \ec{V_k} \nonumber \\
       &+ \gamma^2(H-1) \sum\limits_{j=t_p}^t \sum\limits_{k=t_p}^{j - 1}w_j \frac{6 L}{\alpha} \ec{D_{f} (\hat{x}_k, x_\ast)} \nonumber \\
       &+ \gamma^2(H-1) \sum\limits_{j=t_p}^t \sum\limits_{k=t_p}^{j - 1}w_j \frac{3 \sigmaf^2}{\alpha}.
   \end{align}
   Let us consider the sequence $\{w_k\}_{k=0}^{\infty}$. Recall that $w_k = (1 - \eta)^{-(k+1)}$, where $\eta \eqdef \frac{\gamma\mu}{2\Gamma}$. Then, for $j:\ 0 \leq j \leq H - 1$ with $\gamma \leq \frac{\alpha}{3(H-1)L}$, we obtain
   \begin{align}
       \label{eq:seq-property}
       w_k =&\ (1 - \eta)^{-(k-j+1)}(1 - \eta)^{-j} \leq (1 - \eta)^{-(k-j+1)}(1 + 2\eta)^{j} \nonumber \\
       \leq&\ w_{k - j}\left(1 + \frac{\gamma\mu}{\Gamma}\right)^j \leq w_{k - j}\left(1 + \frac{1}{3(H-1)}\right)^j \nonumber \\ \leq&\  w_{k - j}\exp{\left(\frac{j}{3(H-1)}\right)} \leq w_{k - j}\exp{\left(\frac{1}{3}\right)} \leq 2 w_{k - j}.
   \end{align}
   Using the result above, let us bound terms in \eqref{eq:lma-ide-proof-6}:
   \begin{align}
       \label{eq:abstr-bound-1}
       \sum\limits_{j=t_p}^t \sum\limits_{k=t_p}^{j - 1}w_j \ec{V_k} =&\  \sum\limits_{j=t_p}^t \sum\limits_{k=t_p}^{j - 1}w_{k + (j - k)} \ec{V_k} \overset{\eqref{eq:seq-property}}{\leq} \sum\limits_{j=t_p}^t \sum\limits_{k=t_p}^{j - 1}2w_{k} \ec{V_k} \nonumber \\ =&\ \sum\limits_{j=t_p + 1}^t \sum\limits_{k=t_p}^{j - 1}2w_{k} \ec{V_k} \leq 2(t - t_p) \sum\limits_{k=t_p}^{j - 1}w_{k} \ec{V_k} \nonumber \\
       \leq&\ 2(H-1)\sum\limits_{j=t_p}^t w_{j} \ec{V_j}.
   \end{align}
   \begin{align}
       \label{eq:abstr-bound-2}
       \sum\limits_{j=t_p}^t \sum\limits_{k=t_p}^{j - 1}w_j \ec{D_{f} (\hat{x}_k, x_\ast)} =&\  \sum\limits_{j=t_p}^t \sum\limits_{k=t_p}^{j - 1}w_{k + (j - k)} \ec{D_{f} (\hat{x}_k, x_\ast)} \nonumber \\\overset{\eqref{eq:seq-property}}{\leq}&\ \sum\limits_{j=t_p}^t \sum\limits_{k=t_p}^{j - 1}2w_{k} \ec{D_{f} (\hat{x}_k, x_\ast)} \nonumber \\ =&\ \sum\limits_{j=t_p + 1}^t \sum\limits_{k=t_p}^{j - 1}2w_{k} \ec{D_{f} (\hat{x}_k, x_\ast)} \nonumber \\\leq&\ 2(t - t_p) \sum\limits_{k=t_p}^{j - 1}w_{k} \ec{D_{f} (\hat{x}_k, x_\ast)} \nonumber \\
       \leq&\ 2(H-1)\sum\limits_{j=t_p}^t w_{j} \ec{D_{f} (\hat{x}_j, x_\ast)}.
   \end{align}
   \begin{align}
       \label{eq:abstr-bound-3}
       \sum\limits_{j=t_p}^t \sum\limits_{k=t_p}^{j - 1}w_j \sigmaf^2 =&\  \sum\limits_{j=t_p}^t \sum\limits_{k=t_p}^{j - 1}w_{k + (j - k)} \sigmaf^2 \overset{\eqref{eq:seq-property}}{\leq} \sum\limits_{j=t_p}^t \sum\limits_{k=t_p}^{j - 1}2w_{k} \sigmaf^2 \nonumber \\ =&\ \sum\limits_{j=t_p + 1}^t \sum\limits_{k=t_p}^{j - 1}2w_{k} \sigmaf^2 \leq 2(t - t_p) \sum\limits_{k=t_p}^{j - 1}w_{k} \sigmaf^2 \nonumber \\
       \leq&\ 2(H-1)\sum\limits_{j=t_p}^t w_{j} \sigmaf^2.
   \end{align}
   Substituting \eqref{eq:abstr-bound-1}, \eqref{eq:abstr-bound-2} and \eqref{eq:abstr-bound-3} into \eqref{eq:lma-ide-proof-6}, we can obtain
   \begin{align*}
       \sum\limits_{j=t_p}^t w_j \ec{V_j} \leq&\ \frac{6\gamma^2L^2(H-1)^2}{\alpha^2} \sum\limits_{j=t_p}^t w_j  \ec{V_j} \nonumber \\
       &+  \frac{12\gamma^2(H-1)^2 L}{\alpha}\sum\limits_{j=t_p}^t w_j  \ec{D_{f} (\hat{x}_j, x_\ast)} \nonumber \\
       &+  \frac{6\gamma^2(H-1)^2  \sigmaf^2}{\alpha}\sum\limits_{j=t_p}^t w_j.
   \end{align*}
   Note that we have the same summands in both parts of the inequality. Then,
   \begin{align*}
       \left(1 - \frac{6\gamma^2L^2(H-1)^2}{\alpha^2}\right)\sum\limits_{j=t_p}^t w_j \ec{V_j} \leq&\  \frac{12\gamma^2(H-1)^2 L}{\alpha}\sum\limits_{j=t_p}^t w_j  \ec{D_{f} (\hat{x}_j, x_\ast)} \nonumber \\
       &+  \frac{6\gamma^2(H-1)^2  \sigmaf^2}{\alpha}\sum\limits_{j=t_p}^t w_j.
   \end{align*}
   Since $\gamma \leq \frac{\alpha}{3(H-1)L}$, $1 - \frac{6\gamma^2L^2(H-1)^2}{\alpha^2} \geq \frac{1}{3}$, and we claim
   \begin{align*}
       \sum\limits_{j=t_p}^t w_j \ec{V_j} \leq&\  \frac{36\gamma^2(H-1)^2 L}{\alpha}\sum\limits_{j=t_p}^t w_j  \ec{D_{f} (\hat{x}_j
       , x_\ast)} \nonumber \\
       &+  \frac{18\gamma^2(H-1)^2  \sigmaf^2}{\alpha}\sum\limits_{j=t_p}^t w_j,
   \end{align*}
   which ends the proof of the lemma.
\end{proof}
\begin{lemma} 
    \label{lemma:optimality-gap-single-recursion}
    Suppose that Assumptions~\ref{asm:convexity-and-smoothness},  \ref{asm:finite-sum-stochastic-gradients} and \ref{asm: preconditioner property} hold for Algorithm~\ref{alg:local_sgd_with_preconditioner} which runs for heterogeneous data with $M \geq 2$. Then, for any $\gamma \geq 0$ and $t: t_p \leq t \leq t_{p+1} - 1$ we have
    \begin{align*}
        \label{eq:9f8gff}
        \ecn{\hat{x}_{t+1} - x_\ast}{\hat{D}^{t_p}}\leq&\
        \br{1 - \frac{\gamma \mu}{\Gamma}} \sqn{\hat{x}_{t} - x_\ast}_{\hat{D}^{t_p}} + \frac{\gamma L}{\alpha} \br{1 + \frac{2\gamma L}{\alpha}} V_t \\ &- 2 \gamma \br{1 - \frac{4\gamma L}{\alpha}} D_{f} (\hat{x}_t, x_\ast) + \frac{4 \gamma^2 \sigmaf^2}{M\alpha}.
    \end{align*}   
    In particular, if $\gamma \leq \frac{\alpha}{8L}$, then
    \begin{equation}
        \ecn{\hat{x}_{t+1} - x_\ast}{\hat{D}^{t_p}} \leq \br{1 - \frac{\gamma \mu}{\Gamma}} \sqn{\hat{x}_{t} - x_\ast}_{\hat{D}^{t_p}} + \frac{5\gamma L}{4\alpha} V_t - \gamma D_{f} (\hat{x}_t, x_\ast) + \frac{4 \gamma^2 \sigmaf^2}{M\alpha} \nonumber.
    \end{equation}
\end{lemma}
\begin{proof}
    First we use the update rule $\hat{x}_{t+1} = \hat{x}_t - \gamma (\hat{D}^{t_p})^{-1}g_t$:
    \begin{align*}
        \sqn{\hat{x}_{t+1} - x_\ast}_{\hat{D}^{t_p}} =&\ \sqn{\hat{x}_t - \gamma (\hat{D}^{t_p})^{-1}g_t - x_\ast}_{\hat{D}^{t_p}} \\
        =&\ \sqn{\hat{x}_t - x_\ast}_{\hat{D}^{t_p}} + \gamma^2 \sqn{g_t}_{(\hat{D}^{t_p})^{-1}} - 2 \gamma \ev{\hat{x}_t - x_\ast, g_t} \\
        =&\ \sqn{\hat{x}_t - x_\ast}_{\hat{D}^{t_p}} + \gamma^2 \sqn{g_t}_{(\hat{D}^{t_p})^{-1}} - \frac{2 \gamma}{M} \sum_{m=1}^{M} \ev{\hat{x}_t - x_\ast, \nabla  f_m(x_t^m, z_m)}.
    \end{align*}
    Taking the full expectation and using Lemmas~\ref{lemma:average-gradient-bound} and \ref{lemma:inner-product-bound},
    \begin{align*}
        \ecn{\hat{x}_{t+1} - x_\ast}{\hat{D}^{t_p}}\leq&\ \sqn{\hat{x}_{t} - x_\ast}_{\hat{D}^{t_p}} + \gamma^2 \ecn{g_t}{(\hat{D}^{t_p})^{-1}} \\ &- \frac{2 \gamma}{M} \sum_{m=1}^{M} \ev{\hat{x}_t - x_\ast, \nabla f_m (x_t^m)} \\
        \leq&\ \sqn{\hat{x}_{t} - x_\ast}_{\hat{D}^{t_p}} + \gamma^2 \br{ \frac{2L^2}{\alpha^2} V_t + \frac{8L}{\alpha} D_{f} (\hat{x}_t, x_\ast) + \frac{4 \sigmaf^2}{M\alpha}} \\ &- \frac{2 \gamma}{M} \sum_{m=1}^{M} \ev{\hat{x}_t - x_\ast, \nabla f_m (x_t^m)}  \\
        \leq&\ \br{1 - \frac{\gamma \mu}{\Gamma}} \sqn{\hat{x}_{t} - x_\ast}_{\hat{D}^{t_p}} + \frac{\gamma L}{\alpha} \br{1 + \frac{2\gamma L}{\alpha}} V_t \\ &- 2 \gamma \br{1 - \frac{4\gamma L}{\alpha}} D_{f} (\hat{x}_t, x_\ast) + \frac{4 \gamma^2 \sigmaf^2}{M\alpha}.
    \end{align*}
    This result is the first part of our lemma.
    If $\gamma \leq \frac{\alpha}{8L}$, then $1 - \frac{4 \gamma L}{\alpha} \geq \frac{1}{2}$ and $1 + \frac{2 \gamma L}{\alpha} \leq \frac{5}{4}$, and finally
    \begin{align*}
        \ecn{\hat{x}_{t+1} - x_\ast}{\hat{D}^{t_p}} \leq&\ \br{1 - \frac{\gamma \mu}{\Gamma}} \sqn{\hat{x}_{t} - x_\ast}_{\hat{D}^{t_p}} + \frac{5\gamma L}{4\alpha} V_t - \gamma D_{f} (\hat{x}_t, x_\ast) + \frac{4 \gamma^2 \sigmaf^2}{M\alpha}.
    \end{align*}
\end{proof}
\subsection{\bf Proof of Theorem \ref{thm:hetero_convergence_theorem}}
\begin{proof}
    We start with Lemma~\ref{lemma:optimality-gap-single-recursion}:
    \begin{align*}
        \ecn{\hat{x}_{t+1} - x_\ast}{\hat{D}^{t_p}} \leq&\ \br{1 - \frac{\gamma \mu}{\Gamma}} \sqn{\hat{x}_{t} - x_\ast}_{\hat{D}^{t_p}} + \gamma\left(\frac{5L}{4\alpha} V_t - D_{f} (\hat{x}_t, x_\ast)\right) + \frac{4 \gamma^2 \sigmaf^2}{M\alpha}.
    \end{align*}
    Applying \cref{cor:equal-convergence} and using that $1 + \frac{\gamma\mu}{2\Gamma} \leq 2$, we get
    \begin{align*}
        \ecn{\hat{x}_{t+1} - x_\ast}{\hat{D}^{t+1}} \leq&\ \br{1 - \frac{\gamma \mu}{2\Gamma}} \sqn{\hat{x}_{t} - x_\ast}_{\hat{D}^{t}} + \gamma\left(\frac{5L}{2\alpha} V_t - 2D_{f} (\hat{x}_t, x_\ast)\right) + \frac{8 \gamma^2 \sigmaf^2}{M\alpha}.
    \end{align*}
    Taking the full expectation and summarizing with weights $w_t$, we get
    \begin{align}
        \label{eq:thm-lsgd-dd-proof-1}
        \sum_{t=0}^{T-1} w_t \ecn{\hat{x}_{t+1} - x_\ast}{\hat{D}^{t+1}} \leq&\ \br{1 - \frac{\gamma \mu}{2\Gamma}}\sum_{t=0}^{T-1} w_t \ecn{\hat{x}_{t} - x_\ast}{\hat{D}^{t}} \nonumber \\&+ \gamma \sum_{t=0}^{T-1} \ec{\frac{5L}{2\alpha} w_t V_t - 2 w_t D_{f} (\hat{x}_t, x_\ast)} \nonumber \\ &+ \frac{8 \gamma^2 \sigmaf^2}{M\alpha}\sum_{t=0}^{T-1} w_t .
    \end{align}
    Using that $T = t_p$ for some $p \in \N$, we can decompose the second term and use Lemma~\ref{lemma:iterate-deviation-epoch}:
    \begin{align*}
        \sum_{t=0}^{T-1} w_t \ec{\frac{5L}{2\alpha} V_t - 2 D_{f} (\hat{x}_t, x_\ast)} =&\ \sum_{k=1}^{p} \sum_{t=t_{k-1}}^{t_{k} - 1} 2w_t \ec{\frac{5L}{4\alpha} V_t - D_{f} (\hat{x}_t, x_\ast)} \\
        \leq&\ \sum_{k=1}^{p} \sum_{t=t_{k-1}}^{t_k - 1} 2w_t\br{ \frac{45L^2\gamma^2 (H-1)^2}{\alpha^2}  - 1} \ec{D_{f} (\hat{x}_t, x_\ast)} \\ 
        &+ \frac{45L \gamma^2(H-1)^2\sigmaf^2}{\alpha^2}\sum_{k=1}^{p} \sum_{t=t_{k-1}}^{t_k - 1} w_t.
    \end{align*}
    By assumption on $\gamma$ that $\gamma \leq \frac{\alpha}{10(H-1)L}$ we have $\frac{45L^2\gamma^2 (H-1)^2}{\alpha^2} - 1 \leq -\frac{1}{2}$. Using this, one can obtain
    \begin{align*}
        \sum_{t=0}^{T-1} w_t \ec{\frac{5L}{2\alpha} V_t - 2 D_{f} (\hat{x}_t, x_\ast)} \leq&\ - \sum_{k=1}^{p} \sum_{t=t_{k-1}}^{t_k - 1} w_t\ec{D_{f} (\hat{x}_t, x_\ast)} \\ &+ \frac{45 L \gamma^2(H-1)^2\sigmaf^2 }{\alpha^2} \sum_{k=1}^{p} \sum_{t=t_{k-1}}^{t_k - 1}w_t   \\
        =&\ - \sum_{t=0}^{T-1}w_t \ec{D_{f} (\hat{x}_t, x_\ast)} \\ &+ \frac{45 L \gamma^2(H-1)^2\sigmaf^2}{\alpha^2}\sum_{t=0}^{T-1} w_t.
    \end{align*}
    Plugging this into \eqref{eq:thm-lsgd-dd-proof-1}, we get
    \begin{align*}
        \sum_{t=0}^{T-1} w_t\ecn{\hat{x}_{t+1} - x_\ast}{\hat{D}^{t+1}} \leq&\ \br{1 - \frac{\gamma \mu}{2\Gamma}}\sum_{t=0}^{T-1}w_t \ecn{\hat{x}_{t} - x_\ast}{\hat{D}^{t}} \\ &- \gamma \sum_{t=0}^{T-1} w_t\ec{D_{f} (\hat{x}_t, x_\ast)} \\ &+ \br{\frac{45 L \gamma^3(H-1)^2 \sigmaf^2}{\alpha^2}  + \frac{8 \gamma^2 \sigmaf^2}{M\alpha}}\sum_{t=0}^{T-1} w_t.
    \end{align*}
    Rearranging terms, we have
    \begin{align*}
        \gamma \sum_{t=0}^{T-1} w_t&\ec{D_{f} (\hat{x}_t, x_\ast)} \\ \leq&\ \sum_{t=0}^{T-1} \left(\br{1 - \frac{\gamma \mu}{2\Gamma}}w_t\ecn{\hat{x}_{t} - x_\ast}{\hat{D}^{t}} - w_t \ecn{\hat{x}_{t+1} - x_\ast}{\hat{D}^{t+1}}\right) \\ &+ \br{\frac{45 L \gamma^3(H-1)^2 \sigmaf^2}{\alpha^2}  + \frac{8 \gamma^2 \sigmaf^2}{M\alpha}}\sum_{t=0}^{T-1} w_t.
    \end{align*}
    Noting $W_T \eqdef \sum\limits_{t=0}^T w_t$ and dividing both sides by $\gamma W_{T-1}$, we obtain
    \begin{align*}
        \frac{1}{W_{T-1}} \sum_{t=0}^{T-1} w_t&\ec{D_{f} (\hat{x}_t, x_\ast)} \\ \leq&\ \frac{1}{\gamma W_{T-1}}\sum_{t=0}^{T-1} \left(\br{1 - \frac{\gamma \mu}{2\Gamma}}w_t\ecn{\hat{x}_{t} - x_\ast}{\hat{D}^{t}} - w_t \ecn{\hat{x}_{t+1} - x_\ast}{\hat{D}^{t+1}}\right) \\ &+ \frac{1}{\gamma W_{T-1}}\br{\frac{45 L \gamma^3(H-1)^2 \sigmaf^2}{\alpha^2}  + \frac{8 \gamma^2 \sigmaf^2}{M\alpha}}\sum_{t=0}^{T-1} w_t.
    \end{align*}
    Let us define $\bar{x}_{T-1} \eqdef \frac{1}{W_{T-1}}\sum\limits_{t=0}^{T-1}w_t \hat{x}_t$. Using this definition, definition of the Bregman divergence, the fact that $w_{t-1} = w_t\br{1 - \frac{\gamma \mu}{2\Gamma}}$, the fact that $W_{T-1} \geq w_{T-1} = \br{1 - \frac{\gamma \mu}{2\Gamma}}^{-T}$, boundary for $\gamma$, counting the telescopic sum and applying the Jensen's inequality to the left-hand side, we claim the final result:
    \begin{align*}
        \E\big[f(\bar{x}_{T-1}) - f(x_\ast)\big] \leq&\ \frac{1}{\gamma W_{T-1}}\sum_{t=0}^{T-1} \left(w_{t-1}\ecn{\hat{x}_{t} - x_\ast}{\hat{D}^{t}} - w_t \ecn{\hat{x}_{t+1} - x_\ast}{\hat{D}^{t+1}}\right) \\ &+ \frac{1}{\gamma W_{T-1}}\br{\frac{45 L \gamma^3(H-1)^2 \sigmaf^2}{\alpha^2}  + \frac{8 \gamma^2 \sigmaf^2}{M\alpha}}\sum_{t=0}^{T-1} w_t \\ \leq&\ \br{1 - \frac{\gamma \mu}{2\Gamma}}^T\frac{\sqn{x_{0} - x_\ast}_{\hat{D}^0}}{\gamma} + \gamma\sigmaf^2\br{\frac{9 (H-1) }{2\alpha}  + \frac{8}{M\alpha}} \\ \leq&\ \br{1 - \frac{\gamma \mu}{2\Gamma}}^T\frac{\Gamma\sqn{x_{0} - x_\ast}}{\gamma} + \gamma\sigmaf^2\br{\frac{9 (H-1) }{2\alpha}  + \frac{8}{M\alpha}}. 
    \end{align*}
\end{proof}
\subsection{Proof of \cref{corollary:wc-noniid-unbounded-var}}
\begin{proof}
    We start with \cref{thm:hetero_convergence_theorem}:
    \begin{align*}
        \E\big[f(\bar{x}_{T-1}) - f(x_\ast)\big] \leq&\ \br{1 - \frac{\gamma \mu}{2\Gamma}}^T\frac{\Gamma\sqn{x_{0} - x_\ast}}{\gamma} + \gamma\sigmaf^2\br{\frac{9 (H-1) }{2\alpha}  + \frac{8}{M\alpha}}.
    \end{align*}
    With new notation that $c \eqdef \sigmaf^2\br{\frac{9 (H-1) }{2\alpha}  + \frac{8}{M\alpha}}$, we get
    \begin{align*}
         \E\big[f(\bar{x}_{T-1}) - f(x_\ast)\big] \leq&\ \br{1 - \frac{\gamma \mu}{2\Gamma}}^T\frac{\Gamma\sqn{x_{0} - x_\ast}}{\gamma} + c\gamma \\ \leq&\ \exp{\br{-\frac{\gamma\mu T}{2\Gamma}}}\frac{\Gamma\sqn{x_{0} - x_\ast}}{\gamma} + c\gamma. 
    \end{align*}
    The bound for $\gamma$ from \cref{thm:hetero_convergence_theorem} is $\gamma \leq \frac{\alpha}{10(H-1)L}$. Let us consider two cases:
    \begin{itemize}
        \item $\frac{\alpha}{10(H-1)L} \geq \frac{2\Gamma}{\mu T}\ln{\br{\max{\br{2, \frac{\mu^2 \sqn{x_{0} - x_\ast} T^2}{4\Gamma c}}}}}$.
    \end{itemize}
    Hence, choosing $\gamma = \frac{2\Gamma}{\mu T}\ln{\br{\max{\br{2, \frac{\mu^2 \sqn{x_{0} - x_\ast} T^2}{4\Gamma c}}}}}$, we obtain
    \begin{align*}
        \tilde{\mathcal{O}}&\br{\mu \sqn{x_{0} - x_\ast} T \exp{\br{-\ln{\br{\max{\br{2, \frac{\mu^2 \sqn{x_{0} - x_\ast} T^2}{4\Gamma c}}}}}}}} + \tilde{\mathcal{O}}\br{\frac{\Gamma c}{\mu T}} \\ =&\ \tilde{\mathcal{O}}\br{\frac{\Gamma c}{\mu T}} ,
    \end{align*}
    where in case $2 \geq \frac{\mu^2 \sqn{x_{0} - x_\ast} T^2}{4\Gamma c}$ it holds $\mu \sqn{x_{0} - x_\ast} T \leq \frac{8\Gamma c}{\mu T}$.
    \begin{itemize}
        \item $\frac{\alpha}{10(H-1)L} \leq \frac{2\Gamma}{\mu T}\ln{\br{\max{\br{2, \frac{\mu^2 \sqn{x_{0} - x_\ast} T^2}{4\Gamma c}}}}}$.
    \end{itemize}
    Then, we choose $\gamma$ as $\frac{\alpha}{10(H-1)L}$ Hence, we get
    \begin{align*}
       \frac{10(H-1)L\Gamma\sqn{x_{0} - x_\ast}}{\alpha} &\exp{\br{-\frac{\alpha\mu T}{20\Gamma (H-1)L}}} +\frac{c\alpha}{10(H-1)L} \\ \leq&\ \frac{10(H-1)L\Gamma\sqn{x_{0} - x_\ast}}{\alpha} \exp{\br{-\frac{\alpha\mu T}{20\Gamma (H-1)L}}} \\ &+ \frac{2\Gamma c\ln{\br{\max{\br{2, \frac{\mu^2 \sqn{x_{0} - x_\ast} T^2}{4\Gamma c}}}}}}{\mu T} \\ =&\ \tilde{\mathcal{O}}\br{\frac{(H-1)L\Gamma}{\alpha} \sqn{x_0 - x_\ast} \exp{\br{-\frac{\mu T \alpha}{\Gamma(H-1)L}}}} \\ &+  \tilde{\mathcal{O}}\br{\frac{\Gamma c }{\mu T}}.
    \end{align*}
    Substituting $c$ and combining results above, we have the following estimate:
    \begin{align*}
        &\E\big[f(\bar{x}_{T - 1}) - f(x_\ast)\big] \\ &= \tilde{\mathcal{O}}\br{\frac{(H-1)L\Gamma}{\alpha} \sqn{x_0 - x_\ast} \exp{\br{-\frac{\mu T \alpha}{\Gamma(H-1)L}}} + \frac{\Gamma \sigmaf^2 }{\alpha \mu T} \br{(H-1)   + \frac{1}{M}}},
    \end{align*}
    which finishes the proof of the corollary.
\end{proof}
\end{document}